\documentclass[12pt]{article}
\usepackage{amssymb,amsmath,graphicx,hyperref,amsthm,amsfonts,mathrsfs}
\usepackage{float}
\usepackage{authblk}
\usepackage{algorithm}
\usepackage{algorithmic}

\newcommand{\R}{\mathbb{R}} 
\newcommand{\Q}{\mathbb{Q}}   
\newcommand{\N}{\mathbb{N}}   

\newcommand{\A}{\mathcal{A}} 
\newcommand{\F}{\mathcal{F}} 

\theoremstyle{plain}

\newtheorem{theorem}{Theorem}[section]
\newtheorem{lemma}[theorem]{Lemma}

\newtheorem{corollary}[theorem]{Corollary}

\theoremstyle{definition}
\newtheorem{definition}[theorem]{Definition}
\newtheorem{remark}[theorem]{Remark}

\theoremstyle{remark}
\newtheorem{example}[theorem]{Example}

\numberwithin{equation}{section}
\numberwithin{figure}{section}
\numberwithin{table}{figure}

\title{Borel Isomorphic Dimensionality Reduction of Data and Supervised Learning}
\author{Stan Hatko \\
	shatk094@uottawa.ca \\
	University of Ottawa \\
	Winter 2013 Honours Project \\
	Supervisor: Vladimir Pestov \\
	\vspace{20pt}
}

\begin{document}
\maketitle

\begin{abstract}
In this project we further investigate the idea of reducing the dimensionality of datasets using a Borel isomorphism with the purpose of subsequently applying supervised learning algorithms, as originally suggested by my supervisor V. Pestov (in 2011 Dagstuhl preprint). Any consistent learning algorithm, for example kNN, retains universal consistency after a Borel isomorphism is applied. A series of concrete examples of Borel isomorphisms that reduce the number of dimensions in a dataset is provided, based on multiplying the data by orthogonal matrices before the dimensionality reducing Borel isomorphism is applied. We test the accuracy of the resulting classifier in a lower dimensional space with various data sets. Working with a phoneme voice recognition dataset, of dimension 256 with 5 classes (phonemes), we show that a Borel isomorphic reduction to dimension 16 leads to a minimal drop in accuracy. In conclusion, we discuss further prospects of the method.
\end{abstract}

\clearpage
\newpage
\thispagestyle{empty} 
\mbox{}
\clearpage

\setcounter{section}{-1} 
\section{Some Initial Background in Analysis and Probability}

In this section we have some background definitions and results in probability and analysis that are needed for the reader to understand the introduction. These definitions can be found in standard textbooks, including \cite{Kechris}, \cite{Cohen}, \cite{Stein}, \cite{McDonald}, and \cite{Stewart}.

\begin{definition}\label{D:MetricSpace}
A \emph{metric space} is a nonempty set $X$ with a function $d: X \times X \mapsto \R$ such that 
\begin{enumerate}
\item The distance between a point and itself is zero, the distance from a point to all other points is nonzero:\begin{equation*}
\forall x, y \in X\; d(x, y) = 0 \iff x = y
\end{equation*}
\item The distance is symmetric:\begin{equation*}
\forall x,y \in X\; d(x, y) = d(y, x)
\end{equation*}
\item The \emph{triangle inequality} is satisfied\begin{equation*}
\forall x,y,z \in X\; d(x, z) \leq d(x, y) + d(y, z)
\end{equation*}
\end{enumerate}
\end{definition}

\begin{definition}\label{D:SeparableMetricSpace}
A metric space $X$ is said to be \emph{separable} if it has a countable dense subset $A$, that is for some countable subset $A \subseteq X$,
\begin{equation*}
\forall x \in X\; \forall \epsilon > 0\; \exists y \in A\; d(x, y) \leq \epsilon
\end{equation*}
\end{definition}

\begin{definition}\label{D:SigmaAlgebra}
Let $X$ be a set. Then a $\sigma$-algebra $\mathcal{F} \subseteq \mathcal{P}$ is a subset of the power set of $X$ that satisfies the following properties
\begin{enumerate}

\item $\emptyset \in \mathcal{F}$ ($\mathcal{F}$ contains the empty set)

\item If $A \in \mathcal{F}$, then $A^c \in \mathcal{F}$ ($\mathcal{F}$ is closed under complements)

\item If $A_1, A_2, A_3, ... \in \mathcal{F}$ is a countable family of sets in $\mathcal{F}$, then $\bigcup\limits_{i \in \mathbb{N}} A_{i} \in \mathcal{F}$ ($\mathcal{F}$ is closed under countable unions)

\end{enumerate}
\end{definition}

\begin{definition} \label{D:BorelAlgebra}
The \emph{Borel $\sigma$-algebra} (or simply the \emph{Borel algebra}) on a metric space $X$ is the smallest $\sigma$-algebra containing all open sets in $X$. A set is said to be \emph{Borel} if it is in the Borel algebra. If $X$ is separable, then the Borel algebra is the smallest $\sigma$-algebra containing the open balls of $X$, as shown in Theorem \ref{T:BallsGenerateOpenSets} below.
\end{definition}

\begin{definition} \label{D:Probability}
Let $\Omega$ be a set called the \emph{sample space} and $\mathcal{F}$ be a $\sigma$-algebra on $\Omega$. An \emph{event} $\A \in \mathcal{F}$ is an element of the $\sigma$-algebra (or a subset of $\Omega$ that is in $\mathcal{F}$).  A \emph{probability measure} $P$ on $\mathcal{F}$ is a function $P: \mathcal{F} \mapsto \R$ satisfying
\begin{enumerate}
\item For any event $\A \in \mathcal{F}$, $0 \leq P(\A) \leq 1$.
\item $P(\Omega) = 1$.
\item If $\A_1, \A_2, \A_3, ... \in \mathcal{P}$ is a \emph{countable} family of events such that $A_1$, $A_2$, $A_3$, ... are all pairwise disjoint, then $P(\A_1 \cup \A_2 \cup \A_3 \cup ...) = P(\A_1) + P(\A_2) + P(\A_3) + ...$.
\end{enumerate}
\end{definition}

\begin{lemma} \label{T:PreimagePreservesSetOperations}
If $X$, $Y$ are sets and $f: X \mapsto Y$ is a function, then the inverse image $f^{-1}$ preserves countable set theoretic operations, that is, the following holds:
 for any countable family of sets $A_{i} \subseteq Y$, $i = 1, 2, ...$ we have:
\begin{itemize}
\item $f^{-1}$ preserves countable unions: for any family of sets $A_{i} \subseteq Y$,
\begin{align*}
f^{-1}\left(\bigcup_{i} A_{i}\right) = \bigcup_{i} f^{-1}\left(A_{i}\right)
\end{align*}
\item $f^{-1}$ preserves countable intersections: for any family of sets $A_{i} \subseteq Y$,
\begin{align*}
f^{-1}\left(\bigcap_{i} A_{i}\right) = \bigcap_{i} f^{-1}\left(A_{i}\right)
\end{align*}
\item $f^{-1}$ preserves complements: for any set $A \subseteq Y$,
\begin{align*}
f^{-1}\left(A^{\mathcal{C}}\right) = f^{-1}\left(A\right)^{\mathcal{C}}
\end{align*}
\end{itemize}
\end{lemma}
\begin{proof}
\begin{itemize}
\item $f^{-1}$ preserves countable unions: \begin{align*}
f^{-1}\left(\bigcup_{i} A_{i}\right) &= \left\{x \in X | f(x) \in \bigcup_{i} A_{i} \right\} \\
&= \left\{x \in X | \exists i \; f(x) \in A_{i} \right\} \\
&= \bigcup_{i} \left\{x \in X | f(x) \in A_{i} \right\} \\
&= \bigcup_{i} f^{-1}(A_{i})
\end{align*}

\item $f^{-1}$ preserves countable intersections: \begin{align*}
f^{-1}\left(\bigcap_{i} A_{i}\right) &= \left\{x \in X | f(x) \in \bigcap_{i} A_{i} \right\} \\
&= \left\{x \in X | \forall i \; f(x) \in A_{i} \right\} \\
&= \bigcap_{i } \left\{x \in X | f(x) \in A_{i} \right\} \\
&= \bigcap_{i} f^{-1}(A_{i})
\end{align*}

\item $f^{-1}$ preserves complements: \begin{align*}
f^{-1}\left(A^{\mathcal{C}}\right) &= \left\{x \in X | f(x) \in A^{\mathcal{C}} \right\} \\
&= \left\{x \in X | f(x) \not \in A \right\} \\
&= \left\{x \in X | f(x) \in A \right\}^{\mathcal{C}} \\
&= f^{-1}\left(A\right)^{\mathcal{C}}
\end{align*}
\end{itemize}
\end{proof}

\begin{definition} \label{D:BorelFunction}
Let $X$, $Y$ be metric spaces. A function $f: X \mapsto Y$ is said to be \emph{Borel measurable} (or simply \emph{Borel}) if for any Borel subset of $Y$, $A \subseteq Y$, the inverse image $f^{-1}(A)$ is a Borel set in $X$. Equivalently, a function is Borel measurable if and only if for any open set $V \subseteq Y$, $f^{-1}(V)$ is Borel, as we will show in Theorem \ref{T:generators} below.
\end{definition}

\begin{definition} \label{D:BorelIsomorphism}
Let $X$, $Y$ be metric spaces. A Borel isomorphism is a bijective function $f: X \mapsto Y$ such that both $f$ and $f^{-1}$ are Borel maps.
\end{definition}

\section{Introduction}

\subsection{An Informal Introduction}

Suppose we have a data set consisting of points, where each point is a set of observations, and a set of corresponding responses for each point (where the response can take on a finite number of values). In statistical machine learning, we would like to predict the response for new observations, where we know the data but not the response.

For example, suppose we would like to predict if a person has a predisposition for heart disease based on their genome. We have a database of people with their genome and whether or not they have heart disease. Now suppose we have a person for whom we know their genome but we do not know if they have heart disease. We would like to predict, based on their genomic sequence, if they have heart disease, with the only information available to us being the dataset and the person's genomic sequence. In our example, the response class has two levels (has heart disease or does not have heart disease), so this example is of a \emph{binary} classification problem.

Let $X$ be the data set and $Y$ be the set of responses we are trying to predict. The \emph{classifier} is a function $f: X \mapsto Y$ that attempts to predict the value of the response $y$ for a data entry $x$. The accuracy of $f$ is the proportion of the time that $f$ correctly predicts $y$ for a given $x$. Similarly, the error of $f$ is the proportion of the time that $f$ incorrectly classifies a sample point. We would like to find a classifier $f$ that has an accuracy that is as high as possible (or equivalently, an error that is as small as possible).

There are various types of classifiers, some of which we will discuss below. For any classifier, in order to test its accuracy we take the dataset and split it into two disjoint subsets, the \emph{training set} and the \emph{testing set}. The training set is used in constructing the classifier $f$, and based on this we predict the values for entries in the testing set. We then compare the values we predicted to the actual values and compute the accuracy of our prediction.

There are many classifiers, including k-nearest neighbour (k-NN), support vector machine (SVM), and random forest. The simplest one is k-NN and it is the one we will be using in this project.

The \emph{Bayes error} is the infimum of the errors of all classifiers on the data set. A classifier is said to be \emph{universally consistent} if, as the number of sample points approaches infinity, the error approaches the Bayes error.

\subsection{A More Precise Formulation}

Let $\Omega$ be a metric space called the \emph{domain} (possibly satisfying some additional criteria, such as separability or being a Banach space), $\{1, 2, ..., m\}$ be a finite set, and $\mu$ be a probability measure on $\Omega \times \{1, 2, ..., m\}$.

A \emph{classifier} is a Borel measurable function $f: \Omega \mapsto Y$, that maps points in $\Omega$ to classes in $\{1, 2, ..., m\}$. Without loss of generality, many authors assume that that there are only two classes $\{0, 1\}$. \cite{pbook}

The \emph{misclassification error} of a classifier $f$ is
\begin{align*}
err_{\mu}(f) = \mu \{(x, y) \in \Omega \times \{1, 2, ..., m\} : f(x) \neq y\}
\end{align*}

The \emph{Bayes error} is the infimum of the error of all possible classifiers
\begin{align*}
\ell^{*}(\mu) = \inf_{f} err_{\mu}(f)
\end{align*}

Since the misclassification probability must obviously be in $[0, 1]$, it follows that the Bayes error is well defined for any $\Omega$ (since $[0, 1]$ is bounded below and complete) and is in $[0, 1]$.

Suppose we have a set of $n$ iid random ordered pairs, $(X_1, Y_1)$, $(X_2, Y_2)$, ..., $(X_{n}, Y_{n})$, modelling the data. A classifier constructed from these $n$ data points is denoted $L_{n}$. The process of constructing $L_{n}$ is called $learning$.

A \emph{learning rule} is a sequence of classifiers $L = (L_{n})_{n=1}^{\infty}$, where $L_{n}$ is a classifier constructed from $n$ data points. A learning rule $L$ is said to be \emph{consistent} if $err_{\mu}(L_{n}) = \mu \{(x, y) \in \Omega \times \{1, 2, ..., m\} : L_{n}(x) \neq y\} \to \ell^{*}(\mu)$ in expectation as $n \to \infty$. We say that $L$ is \emph{universally consistent} if $L$ is consistent for every probability measure $\mu$ on $\Omega \times \{1, 2, ..., m\}$. One famous example of a universally consistent classifier (on $\R^{n}$) is the k-Nearest Neighbour classifier (kNN).

\subsection{k-Nearest Neighbour Classifier}

In order to apply the k-Nearest Neighbour classifier (k-NN), we first compute the distance between the input point and all the data points, and we then find the $k$ nearest neighbours of the input point. Then we do a majority vote for which attribute those $k$ neighbours have, and select that as the attribute for the point being classified. Thus, for the k-NN classifier, we require a metric structure on the domain (to compute the distances), and not just a Borel structure.

Here is example pseudocode of k-NN.
\begin{algorithm}
\caption{k-NN pseudocode}
\begin{algorithmic}
\REQUIRE $k \in \N$, $X$ is the domain, $Y$ is the response (must be a finite set $\{1, 2, ..., p\}$), $a \in X$, $(x_1, y_1)$, ..., $(x_{n}, y_{n}) \in X \times Y$
\STATE \COMMENT{Calculate distances from input point to all the data points}
\FOR{$i=1$ \TO $n$}
\STATE $d_{i} \leftarrow d(a, x_{i})$
\ENDFOR
\STATE \COMMENT{Find response for the $k$ nearest neighbours of the input point}
\FOR{$i=1$ \TO $k$}
\STATE $m \leftarrow \underset{m}{\operatorname{arg\,min}} \{ d_{m} $ such that $1 \leq m \leq n$ not previously selected$ \}$
\STATE $a_{i} \leftarrow y_{m}$
\ENDFOR
\STATE \COMMENT{Find number of times each response occurs among the $k$ nearest neighbours}
\FOR{$i=1$ \TO $p$}
\STATE $v_{i} \leftarrow$ number of times $i$ occurs in $\{a_1, a_2, ..., a_{k}\}$
\ENDFOR
\STATE $\{u_1, u_2, ..., u_{q}\} \leftarrow \{ y_{i} | 1 \leq i \leq p$ such that $v_{i}$ is maximal among $v_1, v_2, ..., v_{p} \}$ \COMMENT{Find the most common responses among the $k$ nearest neighbours}
\STATE $s \leftarrow RandomInteger(1, q)$ \COMMENT{Break ties by selecting a random response from the most common responses}
\STATE $r \leftarrow u_{s}$
\RETURN $r$ \COMMENT{Return most common response (or if a tie occurs, one of the most common responses)}
\end{algorithmic}
\end{algorithm}

Stone's theorem states that k-NN on $\R^n$ is universally consistent. \cite{Stone}
\begin{theorem} \label{T:Stone}
The k-NN classifier is universally consistent on $\R^n$ provided $k = k_{n} \rightarrow \infty$ and $\frac{k_{n}}{n} \rightarrow 0$ as $n \rightarrow \infty$. \cite{Stone}
\end{theorem}

One problem is that the data is often very high-dimensional, with many columns of data for each entry. The problem is that finding the nearest $k$ neighbours of a point (as $n \to \infty$) is a computationally hard problem.\cite{Pestov}

An important observation (noted in \cite{Pestov}) is that Stone's theorem does not depend on the Euclidean structure (that is, the metric or topological structure) of the domain, as long as the Borel structure is preserved. This means that if we apply a map to the domain that is bijective and preserves the Borel structure, then kNN will still be universally consistent if applied to the new space that the domain is mapped to. This provides a way to reduce the dimensionality of a data set, as shown in the following section.

\section{Borel Isomorphisms}

\subsection{Borel Sets}

We recall from \ref{D:BorelAlgebra} that the Borel algebra on a metric space $X$ is the smallest $\sigma$-algebra containing all open sets of $X$, and that a set is said to be Borel if it is in the Borel algebra. To proceed futher, we introduce the concept of \emph{generators} of the Borel algebra. This subsection contains some further background results.

\begin{definition}
If $X$ is a metric space, the $\sigma$-algebra \emph{generated} by a set of subsets $\F \subseteq \mathcal{P}(X)$ is the smallest $\sigma$-algebra that contains every set in $\F$. The sets in $\F$ are called the \emph{generators} of this $\sigma$-algebra.
\end{definition}

It is obvious that the set of open subsets of a metric space generates the Borel algebra, since by definition the Borel algebra is the smallest $\sigma$-algebra containing all open sets. To show that a family of subsets $\F \subseteq \mathcal{P}(X)$ generates the Borel algebra, it is sufficient to show that any open set can be created from $\F$ by taking countable unions and complements (recursively) of $\F$, and that all of $\F$ are Borel sets, as shown in the following lemma.

\begin{lemma} \label{L:BorelGenerators}
Let $X$ be a metric space and $\mathcal{F} \subseteq \mathcal{P}(X)$ be a family of Borel subsets of $X$ such that any open set $U \subseteq X$ is in the $\sigma$-algebra generated by $\mathcal{F}$. Then the family of sets $\mathcal{F}$ generates the Borel algebra. 
\end{lemma}
\begin{proof}
Let $\mathcal{A}$ be the $\sigma$-algebra generated by $\mathcal{F}$ and $\mathcal{B}$ be the Borel algebra of $X$.

Since every $B \in \mathcal{F}$ is Borel and the Borel algebra is closed under countable unions and complements, it follows that any set in $\mathcal{A}$ is Borel, so $\mathcal{A} \subseteq \mathcal{B}$.

Since (by assumption) all open subsets of $X$ are in $\mathcal{F}$, all open subsets of $X$ are in $\mathcal{A}$ and therefore $\mathcal{A}$ is a $\sigma$-algebra containing all open subsets of $X$. The Borel algebra $\mathcal{B}$ is the smallest $\sigma$-algebra containing all open subsets of $X$. It follows that $\mathcal{B} \subseteq \mathcal{A}$.

Since $\mathcal{A} \subseteq \mathcal{B}$ and $\mathcal{B} \subseteq \mathcal{A}$, $\mathcal{A} = \mathcal{B}$.
\end{proof}

An important result, that we will use later, is that the open balls of a separable metric space generate the Borel algebra. To prove this we will first need the following result.

\begin{theorem} \label{T:BallsGenerateOpenSets}
If $X$ is a separable metric space and $A \subseteq X$ is an open set, then $A$ is the countable union of open balls.
\end{theorem}
\begin{proof}
Let $Y$ be a countable dense subset of $X$, and $I = A \bigcap Y$. Then $I$ is a countable dense subset of $A$.

If $A = X$, then it is obvious that $A = \bigcup\limits_{x \in I} B_{1}(x)$, which is a countable union of open balls (of radius $1$).

If $A \neq X$, then $X \setminus A$ is nonempty (so $\exists w \in X \setminus A$). Let \begin{align*}
F = \{B_{\epsilon}(x) : x \in I, \epsilon = \inf\{d(x, y) : y \in X \setminus A\}\}
\end{align*}

Since $I$ is countable, it follows that $F$ is a countable set of open balls.

For any $B_{\epsilon}(x) \in F$, $\forall y \in B_{\epsilon}(x)$, $y \in A$, so it follows that $B_{\epsilon}(x) \subseteq A$. Since $F$ is a set of open balls in $A$, it follows that the union of the open balls $\cup F$ is a subset of $A$.

Let $x \in A$ and $\epsilon_{x} = \inf\{d(x, y) : y \in X \setminus A\}$. Since $I$ is dense in $A$, for some $z \in B_{\epsilon_{x} / 2}$, $z \in I$. Then for any $y \in X \setminus A$,
\begin{align*}
\epsilon_{x} &\leq d(x, y) \\
&\leq d(x, z) + d(z, y) \\
&\leq \frac{\epsilon_{x}}{2} + d(z, y)
\end{align*}

It follows that, for all $y \in X \setminus A$,
\begin{align*}
d(z, y) \geq \frac{\epsilon_{x}}{2}
\end{align*}

Let $\epsilon_{z} = \inf\{d(z, y) : y \in X \setminus A\}$ ($\epsilon_{z}$ is the radius of the ball around $z$ in the set $F$). Since $\forall y \in X \setminus A$, $d(z, y) \geq \frac{\epsilon_{x}}{2}$, it follows that the $\epsilon_{z} = \inf\{d(z, y) : y \in X \setminus A\} \geq \frac{\epsilon_{x}}{2}$. 

Since $d(x, z) < \frac{\epsilon_{x}}{2} < \epsilon_{z}$, it follows that $x \in B_{\epsilon_{z}}(z)$, and hence $x$ is contained in at least one of the open balls in $F$. Therefore, $x \in \bigcup F$. This means that for all $x \in F$, $x \in \bigcup F$, and hence $F \subseteq X$.

Therefore, there is a countable set $F$ of open balls such that $F = X$ (since $F \subseteq X$ and $X \subseteq F$).
\end{proof}

We are now able to show that the Borel algebra of a separable metric space is generated by the open balls of that metric space.

\begin{theorem} \label{T:BallsGenerateBorelAlgebra}
The Borel algebra of a separable metric space $X$ is generated by the family of open balls of $X$.
\end{theorem}
\begin{proof}
Let $A \subseteq X$ be an open set. Then by Theorem \ref{T:BallsGenerateOpenSets}, it follows that $A$ is the union of a countable family of open balls. This means that under the operation of countable unions, any open set can be formed from open balls. This means that the set of open balls generates the $\sigma$-algebra containing all open sets, and therefore generates the Borel algebra.
\end{proof}

\begin{lemma} \label{T:LowerBoundedIntervalsGenerateBorelAlgebra}
The set of upper open rays $(a, \infty)$ (where $a \in \R$) generates the Borel $\sigma$-algebra on $\R$. \cite{Folland}
\end{lemma}
\begin{proof}
Let $B_{r}(x)$ be an open ball in $\R$ centered at $x \in \R$ of radius $r$. Then \begin{align*}
B_{r}(x) = \left(x - r, x + r\right)
\end{align*}

We then define $I_{n}$ (for $n \geq 1$) as 
\begin{align*}
I_{n} = \left(x - \frac{n}{n+1} r, x + \frac{n}{n+1} r \right]
\end{align*}

We find that
\begin{align*}
I_{n} &= \left(x - \frac{n}{n+1} r, x + \frac{n}{n+1} r \right] \\
&= \left(x - \frac{n}{n+1} r, \infty\right) \bigcap \left( -\infty, x + \frac{n}{n+1} r \right] \\
&= \left(\left(\left(x - \frac{n}{n+1} r, \infty\right) \bigcap \left( -\infty, x + \frac{n}{n+1} r \right]\right)^{\mathsf{C}}\right)^{\mathsf{C}} \\
&= \left(\left(x - \frac{n}{n+1} r, \infty\right)^{\mathsf{C}} \bigcup \left( -\infty, x + \frac{n}{n+1} r \right]^{\mathsf{C}}\right)^{\mathsf{C}} \\
&= \left(\left(x - \frac{n}{n+1} r, \infty\right)^{\mathsf{C}} \bigcup \left( x + \frac{n}{n+1} r, \infty \right)\right)^{\mathsf{C}} \\
\end{align*}

This means that for all $n \geq 1$, $I_{n}$ can we written in terms of complements and countable unions (in fact, finite unions) of upper open rays.

We would like to to show that $\bigcup\limits_{n = 1}\limits^{\infty} I_{n} = B_{r}(x)$ by proving subset inclusions both ways.

Let $n \geq 1$ and $z \in I_{n}$. Then \begin{align*}
x - r < x - \frac{n}{n+1} r < z \leq x + \frac{n}{n+1} r < x + r
\end{align*}
This means that $z \in (x - r, x + r) = B_{r}(x)$, which means that $I_{n} \subseteq B_{r}(x)$. Since for all $n \geq 1$, $I_{n} \subseteq B_{r}(x)$, it follows that the union of all the $I_{n}$ is a subset of $B_{r}(x)$, \begin{align*}
\bigcup\limits_{n = 1}\limits^{\infty} I_{n} \subseteq B_{r}(x)
\end{align*}

Let $z \in B_{r}(x)$ and let $s = |x - z|$, so that $z = x \pm s$. Since $B_{r}(x)$ is an open ball of radius $r$, it follows that $ 0 \leq s < r$. We then set \begin{align*}
n > \frac{s}{r - s}
\end{align*}

We then find that
\begin{align*}
x + \frac{n}{n+1} r > x + \frac{\frac{s}{r - s}}{\frac{s}{r - s}+1} r = x + \frac{s}{s + r - s} r = x + s 
\end{align*}
and that
\begin{align*}
x - \frac{n}{n+1} r < x - \frac{\frac{s}{r - s}}{\frac{s}{r - s}+1} r = x - \frac{s}{s + r - s} r = x - s 
\end{align*}

This means that $x - s, x + s \in \left(x - \frac{n}{n+1} r, x + \frac{n}{n+1} r\right]$, which means that $z \in I_{n}$, so $z \in \bigcup\limits_{n = 1}\limits^{\infty} I_{n}$. It follows that $B_{r}(x) \subseteq \bigcup\limits_{n = 1}\limits^{\infty} I_{n}$.

Therefore, since $\bigcup\limits_{n = 1}\limits^{\infty} I_{n} \subseteq B_{r}(x)$ and $B_{r}(x) \subseteq \bigcup\limits_{n = 1}\limits^{\infty} I_{n}$, $B_{r}(x) = \bigcup\limits_{n = 1}\limits^{\infty} I_{n}$.
\end{proof}

\subsection{Borel Functions and Isomorphisms}

We recall the definition of a Borel function and isomorphism from \ref{D:BorelFunction} and \ref{D:BorelIsomorphism}, respectively. A function $f: X \mapsto Y$ is a Borel isomorphism if and only if it is bijective, $f$ maps Borel sets to Borel sets, and $f^{-1}$ maps Borel sets to Borel sets. 

\begin{theorem} \label{T:composition} 
The composition of Borel isomorphisms is a Borel isomorphism.
\end{theorem}

\begin{proof}
Let $X$, $Y$, $Z$ be metric spaces and $f: X \mapsto Y$, $g: Y \mapsto Z$ be Borel isomorphisms.
We need to show that $g \circ f: X \mapsto Z$ is a Borel isomorphism, so that $g \circ f$ is bijective, maps Borel sets to Borel sets, and whose inverse maps Borel sets to Borel sets.
\begin{itemize}
\item Show that $g \circ f$ is bijective.

Follows from the fact that the composition of bijections is a bijection.

\item Show that $f \circ g$ maps Borel sets to Borel sets.

Let $B \subseteq X$ be a Borel set. Since $f: X \mapsto Y$ is a Borel isomorphism, it follows that $f(B) \subseteq Y$ is a Borel set. Furthermore, $g: Y \mapsto Z$ is a Borel isomorphism, so $g(f(B)) \subseteq Z$ is a Borel set, which means that $g \circ f(B) \subseteq Z$ is Borel. Therefore $g \circ f: X \mapsto Z$ maps Borel sets to Borel sets.

\item Show that the inverse of $g \circ f$ is maps Borel sets to Borel sets.

Since both $f: X \mapsto Y$ and $g: Y \mapsto Z$ are bijective, the inverse of $g \circ f: X \mapsto Z$ is $f^{-1} \circ g^{-1}: Z \mapsto X$. Let $B \subseteq Z$ be a Borel set. Since $g: Y \mapsto Z$ is a Borel isomorphism, the inverse $g^{-1}: Z \mapsto Y$ is a Borel isomorphism and so $g^{-1}(B) \subseteq Y$ is Borel. Furthermore $f: X \mapsto Y$ is also a Borel isomorphism, so that the inverse $f^{-1}: Y \mapsto Z$ is a Borel isomorphism and so $f^{-1}(g^{-1}(B)) \subseteq Y$ is also a Borel set. Therefore the inverse maps Borel sets to Borel sets.
\end{itemize}

Since $g \circ f: X \mapsto Z$ is bijective, maps Borel sets to Borel sets, and its inverse maps Borel sets to Borel sets, it is a Borel Isomorphism.
\end{proof}

This means that if we find some Borel isomorphisms, their composition is a Borel isomorphism. 

In general, proving that a function is Borel can be very difficult, as we must show that the preimage of any arbitrary Borel set is a Borel set. Fortunately, we have a theorem that provides an easier way to verify that a function is Borel, by allowing us to check that the inverse image of a generator is Borel instead of having to check every Borel set. In order to prove this theorem we must first prove a couple of lemmas.

\begin{lemma} \label{T:PreimageIsSigmaAlgebra}
Let $X$, $Y$ be metric spaces and $\mathcal{F}$ be a $\sigma$-algebra on $Y$. Let $f: X \mapsto Y$ be a function. Then $\mathcal{A} = \{f^{-1}(A) : A \in \F\}$ is a $\sigma$-algebra.
\end{lemma}
\begin{proof}
Since $\F$ is a $\sigma$-algebra, $\emptyset \in \F$ and $Y \in \F$, which means that $f^{-1}(\emptyset) = \emptyset \in \A$ and $f^{-1}(Y) = X \in \A$, so $\A$ contains the empty set and the entire space.

Let $A_{i} \in \F$ ($i \in \N$) be a countable family of sets in $\F$ and $f^{-1}(A_{i}) \in \A$ be the corresponding countable family of sets in $\A$. Then by Lemma \ref{T:PreimagePreservesSetOperations}, $\bigcup\limits_{i \in \N} f^{-1}(A_{i}) = f^{-1}\left(\bigcup\limits_{i \in \N} A_{i}\right)$, and since $\F$ is a $\sigma$-algebra $\bigcup\limits_{i \in \N} A_{i} \in \F$, which means that $f^{-1}\left(\bigcup\limits_{i \in \N} A_{i}\right) \in \A$, and hence $\A$ is closed under countable unions.

Let $A \in \F$, so $f^{-1}(A) \in \A$. Then by Lemma \ref{T:PreimagePreservesSetOperations}, $f^{-1}(A)^\mathcal{C} = f^{-1}(A^\mathcal{C})$, which means $f^{-1}(A)^\mathcal{C} \in \A$. This means that $\A$ is closed under complements.

Since $\A$ contains $\emptyset$ and $X$, and is closed under countable unions and complements, $\A$ is a $\sigma$-algebra.
\end{proof}

\begin{lemma} \label{T:PreimageSigmaAlgebraCorollary}
Let $X$, $Y$ be metric spaces, $f: X \mapsto Y$ be a function, and $\mathcal{A}$ and $\mathcal{B}$ be $\sigma$-algebras on $X$ and $Y$, respectively. Then $\F = \{B \in \mathcal{B} : f^{-1}(B) \in \A\}$ is a $\sigma$-algebra.
\end{lemma}
\begin{proof}
We notice that both $\emptyset$, $Y \in \mathcal{B}$, $f^{-1}(\emptyset) = \emptyset \in \A$, and $f^{-1}(Y) = X \in \A$, so both $\emptyset$, $Y \in \F$.

For any countable family $B_{i} \in \F$, since $\mathcal{B}$ is a $\sigma$-algebra $\bigcup\limits_{i \in \N} B_{i} \in \mathcal{B}$ and $f^{-1}\left(\bigcup\limits_{i \in \N} B_{i}\right) = \bigcup\limits_{i \in \N} f^{-1}(B_{i}) \in \A$ (by Lemmas \ref{T:PreimagePreservesSetOperations} and \ref{T:PreimageIsSigmaAlgebra}), which means that $\bigcup\limits_{i \in \N} B_{i} \in \F$.

For any set $B \in \F$, since $\mathcal{B}$ is a $\sigma$-algebra $B^{\mathcal{C}} \in \mathcal{B}$ and $f^{-1}\left(B^{\mathcal{C}}\right) = f^{-1}\left(B\right)^{\mathcal{C}} \in \A$ (by Lemmas \ref{T:PreimagePreservesSetOperations} and \ref{T:PreimageIsSigmaAlgebra}), which means that $B^{\mathcal{C}} \in \F$.

This means that $\F$ contains $\emptyset$ and $Y$, is closed under countable unions, and is closed under complements. Therefore, $\F$ is a $\sigma$-algebra on $Y$.
\end{proof}

\begin{theorem} \label{T:generators}

Let $X$, $Y$ be metric spaces, $f: X \mapsto Y$ be a function, $\mathcal{A}$ and $\mathcal{B}$ be the Borel algebras on $X$ and $Y$, respectively, and $\mathcal{C}$ be a family of sets that generates the Borel algebra of $Y$. If for all $B \in \mathcal{C}$, $f^{-1}(B)$ is Borel, then $f$ is a Borel function. \cite{Berberian} \cite{Dudley}
\end{theorem}
\begin{proof}
Let $\mathcal{D} = \{B \subseteq \mathcal{B} : f^{-1}(B) \in A\}$. Then by Lemma \ref{T:PreimageSigmaAlgebraCorollary} $\mathcal{D}$ is a $\sigma$-algebra.

Since $\forall B \in \mathcal{D}$, $B \in \mathcal{B}$, it is clear that $\mathcal{D} \subseteq \mathcal{B}$.

For any set $B \in \mathcal{C}$, $B \in \mathcal{B}$ and $f^{-1}(B) \in \A$ by assumption, which means that $B \in \mathcal{D}$. It follows that $\mathcal{C} \subseteq \mathcal{D}$. Since $\mathcal{B}$ is the smallest $\sigma$-algebra containing $\mathcal{C}$ (since $\mathcal{C}$ generates $\mathcal{B}$) and $\mathcal{D}$ is a $\sigma$-algebra containing $\mathcal{C}$, this means that $\mathcal{B} \subseteq \mathcal{D}$.

Since $\mathcal{D} \subseteq \mathcal{B}$ and $\mathcal{B} \subseteq \mathcal{D}$, it follows that $\mathcal{D} = \mathcal{B}$. This means that for any set $B \in \mathcal{B}$, $f^{-1}(B) \in \A$ and hence $f$ is Borel.
\end{proof}

\begin{theorem} \label{T:continuous} 
Any continuous function is Borel.
\end{theorem}
\begin{proof}
Let $X$, $Y$ be metric spaces, and $f: X \mapsto Y$ be a continuous function. Then for any open set $A \subseteq Y$, $f^{-1}(A)$ is open and is therefore Borel. Since the open sets of $Y$ generate the Borel algebra of $Y$ and the inverse image of any open set is a Borel set, by Theorem \ref{T:generators}, $f$ is Borel.
\end{proof}

We use this fact to show that any homeomorphism is a Borel isomorphism.

\begin{theorem} \label{T:homeomorphism} 
Any homeomorphism (continuous bijection with a continuous inverse) is a Borel isomorphism.
\end{theorem}
\begin{proof}
Let $X$, $Y$ be metric spaces and $f: X \mapsto Y$ be a continuous bijection with a continuous inverse.
We need to show that $f: X \mapsto Y$ is a Borel isomorphism.

We already know $f: X \mapsto Y$ is bijective. Furthermore, since $f$ is a bijective continuous function, by Theorem \ref{T:continuous} it follows that $f$ is a Borel function. Furthermore, $f^{-1}$ is also a continuous bijective function, so by Theorem \ref{T:continuous} $f^{-1}$ is also a Borel function. Since $f: X \mapsto Y$ is a bijective Borel function whose inverse is Borel, it follows that $f: X \mapsto Y$ is a Borel isomorphism.
\end{proof}

\begin{theorem} \label{T:PointwiseBorel}
Let $X$, $Y$ be metric spaces and $f_n: X \mapsto Y$ be a sequence of Borel functions that converges pointwise to a function $f: X \mapsto Y$. Then $f$ is Borel. \cite{Dudley}
\end{theorem}
\begin{proof}
The following proof is based on the proof of Theorem 4.2.2 in \cite{Dudley}, but with more details added. \footnote{A simpler version of the proof is possible if we assume $Y = \R$, but our version works when $Y$ is an arbitrary metric space. See \cite{Bruckner} for the simpler proof assuming $Y = \R$.}

We need to show that $f^{-1}(U)$ is Borel for any open set $U \subseteq Y$, since the open sets generate the Borel algebra by Theorem \ref{T:generators} this is sufficient to show $f$ is Borel.

Let $F_{m} = \{y \in U | B_{1/m}(y) \subseteq U\}$. We need to show that $F_{m}$ is closed. Let $x_1, x_2, ...$ be a sequence in $F_{m}$ that converges to $x \in Y$. Suppose that $x \not \in F_{m}$.

Let $\ell = \inf\{d(x, y) | y \not \in U$\}, since $x \not \in F_{m}$, $\ell < \frac{1}{m}$. Set $\epsilon < \frac{1}{m} - \ell$. Then there exists an $N \geq 1$ such that for all $n \geq N$, $d(x_{n}, x) < \epsilon$. We then find that for all $n \geq N$,
\begin{align*}
\inf\{ d(x_{n}, y) | y \not \in U \} &\leq \inf\{ d(x_{n}, x) + d(x, y) | y \not \in U \}\\
&\leq \epsilon + \ell \\
&< \frac{1}{m} - \ell + \ell \\
&= \frac{1}{m}
\end{align*}
This means that $\inf\{ d(x_{n}, y) | y \not \in U \} < \frac{1}{m}$, and so for some $r < \frac{1}{m}$ there exists $y \not \in U$ such that $d(x_{n}, y) \leq r < \frac{1}{m}$, so that $B_{1/m}(x_{n}) \not \subseteq U$, which is a contradiction as $x_1, x_2, ... \in F_{m}$. Therefore, $x \in F_{m}$ and so $F_{m}$ is closed.

It is obvious that if $x \in F_{m}$ (for some $m$), then $x \in U$. If $x \in U$, since $U$ is open for some $\epsilon > 0$, $B_{\epsilon}(x) \subseteq U$, which means that for $M > \frac{1}{\epsilon}$, for all $m > M$, $B_{1/m}(x) \subseteq U$. This means that $x \in U$ if and only if $\exists M \geq 1$ $\forall m \geq M$ $x \in F_{m}$.

Since $f_{n}$ converges pointwise to $f$, it follows that for all $x \in f^{-1}(U)$, there exists an $N \geq 1$ such that for all $n > N$, $d(f(x), f_{n}(x)) < \frac{1}{2 m}$ which means that for all $n \geq N$, $f_{n}(x) \in F_{2 m}$. It follows that for some $m \geq 1$ and $N \geq 1$, if $f(x) \in f^{-1}(U)$ then $f(x) \in \bigcap\limits_{n \geq N} f_{n}(F_{m})$.

Similarly, if there exists an $N \geq 1$ such that for all $n \geq N$, $f_{n}(x) \in F_{m}$, then it is obvious that $f(x) \in F_{m}$ and hence $f(x) \in U$. This means that if for some $m \geq 1$ and $n \geq N$, $f(x) \in \bigcap\limits_{n \geq N} f_{n}(F_{m})$, then $f(x) \in f^{-1}(U)$.

Combining these results, we find that
\begin{align*}
f^{-1}(U) = \bigcup\limits_{m=1}\limits^{\infty} \bigcup\limits_{N=1}\limits^{\infty} \bigcap\limits_{n \geq N} f_{n}(F_{m})
\end{align*}
\end{proof}

\begin{theorem} \label{T:BorelBijectionIsIsomorphism}
Let X, Y be complete separable metric spaces and $f: X \mapsto Y$ be a Borel bijective function. Then $f$ is a Borel isomorphism.
\end{theorem}
\begin{proof}
This result is proven in Theorem 14.12, Theorem 15.1 (the Lusin-Souslin Theorem) and Corollary 15.2 in \cite{Kechris}.
\end{proof}

\begin{theorem}
Any two complete separable metric spaces of the same cardinality are Borel isomorphic. \cite{Kechris}
\end{theorem}
\begin{proof}
The following sketch of a proof is based on Theorem 15.6 of \cite{Kechris}. We omit the technical details.

If both metric spaces $X$ and $Y$ are countable sets of the same cardinality, then the Borel algebra on $X$ and $Y$ is simply the power set, so that any bijection between $X$ and $Y$ is a Borel isomorphism.

We therefore need to show that any two uncountable separable complete metric spaces are Borel isomorphic. Let $X$ be an uncountable complete separable metric space. We need to show that $X$ is Borel isomorphic to the Cantor set $\mathcal{C}$.

We start with the fact that the closed unit inverval $I = [0, 1]$ and $\mathcal{C}$ are Borel isomorphic, in fact, here a Borel isomorphism between $I$ and $\mathcal{C}$ is can be constructed explicitly, by first taking the binary expansion of $x \in I$ and then in the binary expansion replacing each digit $1$ with with the digit $2$, which is then the ternary (base $3$) expansion of a value in the Cantor set (in the binary representation here, all numbers other than $1$ are expressed without an infinite string of $1$'s, and the number $1$ is expressed as $0.111...$). \cite{Srivastava} We then use the fact that I and the Hilbert cube $I^{\N}$ are Borel isomorphic, which means that $\mathcal{C}$ and $I^{\N}$ are Borel isomorphic. By Theorem 4.14 in Kechris, $X$ is homeomorphic to a subspace of $I^{\N}$, which means that there is an injective Borel mapping from $X$ to $\mathcal{C}$.  Also, by Theorem 6.5 in Kechris $X$ there is a continuous injective function from $\mathcal{C}$ to $X$.

This means that there is a Borel injective mapping $f: X \mapsto \mathcal{C}$ and a Borel injective mapping $g: \mathcal{C} \mapsto X$. Therefore, by the Borel Schr\"oder-Bernstein Theorem (Theorem 15.7 in Kechris), there exists a Borel isomorphism between $X$ and $\mathcal{C}$. Since any uncountable complete separable metric space is Borel isomorphic to the Cantor set $\mathcal{C}$, this means that any two uncountable complete separable metric spaces are Borel isomorphic.
\end{proof}

This means all uncountable separable complete metric spaces are Borel isomorphic, for instance the Cantor set, the closed unit interval $[0, 1]$, the real line $\R$, Euclidean space $\R^{n}$, $\ell^2$ space, the space of continuous functions on $[0, 1]$ (with the supremum metric) $C[0, 1]$, the product of finitely many unit intervals (with the Euclidean metric) $[0, 1]^{n}$, and any separable Banach space other than $\{0\}$ are all pairwise Borel isomorphic.

\subsection{Borel Maps to $\R$}

Here we work out in detail a particular case, where the function is a mapping to $\R$, as an example that will be used later in the project. For functions to $\R$ we have a theorem that provides a simpler way to verify that a function is Borel.

\begin{theorem} \label{T:BorelIfPreimageOfIntervalIsBorel}
A function $f: X \mapsto \R$ is Borel if and only if for all $a \in \R$, $f^{-1}((a, \infty))$ is Borel.
\end{theorem}
\begin{proof}
Since for all $a \in \R$, $(a, \infty)$ is open (and therefore Borel), the forward direction is obvious.

By Lemma \ref{T:LowerBoundedIntervalsGenerateBorelAlgebra}, the upper open rays of $\R$ generate the Borel algebra of $\R$. Therefore, by Theorem \ref{T:generators}, since the inverse image of any lower-bounded open interval is Borel and the lower-bounded open intervals generate the Borel algebra, $f$ is Borel.
\end{proof}

\begin{remark} \label{R:BorelIfPreimageOfIntervalIsBorelExtended}
Since the Borel algebra on $\R$ is also generated by the lower open rays $(-\infty, a)$ (where $a \in \R$), the open intervals $(a, b)$ (where $a, b \in \R$), and the closed intervals extending to infinity (closed rays) $(-\infty, a]$ and $[a, \infty)$ (where $a \in \R$), the inverse image of $(a, \infty)$ in Theorem \ref{T:BorelIfPreimageOfIntervalIsBorel} can be replaced by any of these and the theorem remains valid. \cite{Folland}
\end{remark}

Using this theorem, it becomes much easier to prove a function is Borel, as we only have to show that the preimage of $(a, \infty)$ for any $a \in \R$ is Borel, instead of having to show that the preimage of any Borel set is Borel.

\begin{lemma} \label{L:FlipBorel}
If $f: X \mapsto \R$ is Borel, then $-f$ is Borel.
\end{lemma}
\begin{proof}
Let $a \in \R$. Since $f$ is Borel, $f^{-1}((a, \infty))$ is a Borel set. Then we find
\begin{align*}
(-f)^{-1}((a, \infty)) &= \{x | -f(x) > a\} \\
&= \{x | f(x) < -a\} \\
&= f^{-1}((-\infty, -a))
\end{align*}

Since $f$ is Borel and $(-\infty, -a)$ is a Borel set, $f^{-1}((-\infty, -a))$ is a Borel set, which means that $(-f)^{-1}((a, \infty))$ is a Borel set, and hence $-f$ is Borel.
\end{proof}

\begin{theorem} \label{T:MultiplyScalarBorel}
If $f: X \mapsto \R$ is Borel, then for all $c \in \R$, $c f$ is Borel.
\end{theorem}
\begin{proof}
Let $a \in \R$. Since $f: X \mapsto \R$ is Borel, this means $f^{-1}((a, \infty))$ is Borel.

Suppose without loss of generality that $c > 0$ (if $c = 0$, then $f$ is the zero function, which is obviously Borel, and if $c < 0$, then by Lemma \ref{L:FlipBorel}, $-f$ is Borel and so $c f = (-c) (-f)$, which is then a positive constant multiplied by a Borel function). Then
\begin{align*}
(c f)^{-1} ((a, \infty)) &= \{x | c f(x) > a\} \\
&= \left\{x | f(x) > \frac{a}{c}\right\} \\
&= f^{-1} \left(\left(\frac{a}{c}, \infty\right)\right)
\end{align*}
Since $f$ is Borel and $\frac{a}{c} \in \R$, $f^{-1} \left(\left(\frac{a}{c}, \infty\right)\right)$ is a Borel set, and so $(c f)^{-1} ((a, \infty))$ is a Borel set, and so $c f$ is Borel.
\end{proof}

\begin{theorem} \label{T:BorelSum}
The sum of two Borel functions is Borel, that is, if $f, g: X \mapsto \R$ are Borel functions, then $f + g$ is Borel. \cite{Strichartz}
\end{theorem}
\begin{proof}
Elements of the following proof are based on \cite{Strichartz}.

Let $h = f + g$. We need to show that $h^{-1}((a, \infty))$ is Borel for all $a \in \R$.

We see that, for any $a \in \R$, $h^{-1}((a, \infty)) = \{x \in X | f(x) + g(x) > a \}$.

Let $x \in X$ such that $f(x) + g(x) > a$. Let $b \in \Q$ such that $f(x) > b$ and $g(x) > a - b$.

We have that
\begin{align*}
h^{-1}((a, \infty)) &= \bigcup\limits_{b \in \Q} \{x \in X | f(x) > b \wedge g(x) > a - b\} \\
&= \bigcup\limits_{b \in \Q} \left(\{x \in X | f(x) > b\} \bigcap \{x \in X | g(x) > a - b\}\right) \\
&= \bigcup\limits_{b \in \Q} \left(f^{-1}((b, \infty)) \bigcap g^{-1}((a - b, \infty))\right)
\end{align*}
Since both $f$ and $g$ are Borel, both $f^{-1}((b, \infty))$ and $g^{-1}((a - b, \infty))$ are Borel sets, and since $h^{-1}((a, \infty))$ is the countable union (since over $\Q$) of the intersection of these sets, it is a Borel set. Therefore, by Theorem \ref{T:BorelIfPreimageOfIntervalIsBorel}, $h$ is Borel.
\end{proof}

\begin{theorem} \label{T:BorelLinearCombination}
Any finite linear combination of Borel functions from $X \subseteq \R$ to $\R$ is Borel.\cite{Strichartz}
\end{theorem}
\begin{proof}
Let $f, g: X \mapsto \R$ be Borel functions and $a, b \in \R$.

Since $f$ and $g$ are Borel functions, by Theorem \ref{T:MultiplyScalarBorel} $a f$ and $b g$ are Borel, and then by Theorem \ref{T:BorelSum} $a f + b g$ is Borel.

From this, it follows by induction that higher finite linear combinations (ie. $a f + b g + c h$) are also Borel.
\end{proof}

\begin{theorem} \label{T:BorelSeries}
Let $g_{n}: \R \mapsto \R$ be a series of Borel functions such that their sum converges pointwise to a function $f: \R \mapsto \R$,
\begin{align*}
\forall x \in \R, \;f(x) = \sum_{n=0}^{\infty} g_{n}(x)
\end{align*}
Then $f: \R \mapsto \R$ is a Borel function.
\end{theorem}
\begin{proof}
Let $f_{k}: \R \mapsto \R$ be a sequence of functions defined as
\begin{align*}
f_{k}(x) = \sum_{n=0}^{k} g_{n}(x)
\end{align*}

We notice that for all $k \geq 0$, $f_{k}$ is the sum of a finite number of Borel functions ($g_0$, $g_1$, ..., $g_{k}$), and so by Theorem \ref{T:BorelSum} it follows that $f_{k}$ is Borel.

We also notice that $f(x) = \lim\limits_{k \to \infty} f_{k}(x)$. This means that $(f_{k})_{k=0}^{\infty}$ is a sequence of Borel functions that converges pointwise to a function $f$, and so by Theorem \ref{T:PointwiseBorel} $f$ is Borel.
\end{proof}

\begin{theorem} \label{T:monotone}
Any monotone increasing function $f: \R \mapsto \R$ is Borel.
\end{theorem}
\begin{proof}
Let $f: \R \mapsto \R$ be a monotone increasing function, and let $r \in \R$. There are three possible cases for $f^{-1}((r, \infty))$:
\begin{enumerate}
\item For all $x \in \R$, $f(x) \leq r$. Then $f^{-1}((r, \infty)) = \emptyset$, which by definition is a Borel set.

\item For all $x \in \R$, $f(x) > r$. Then $f^{-1}((r, \infty)) = \R$, which by definition is a Borel set.

\item There exists an $x \in \R$ such that $f(x) > r$ and a $y \in \R$ such that $f(y) \leq r$. Then let $L = \inf\{x \in \R\; |\; f(x) > r\}$, since the set of all $x \in \R$ such that $f(x) > r$ is bounded below (by $y$) the infimum exists. Then for all $x > L$, $f(x) > r$, and for all $x < L$, $f(x) \leq r$. This means that if $f(L) > r$, then $f^{-1}((a, b)) = \{x \in \R | x \geq L\} = [x, \infty)$, which is a Borel set, and if $f(L) \leq r$, then $f^{-1}((a, b)) = \{x \in \R | x > L\} = (x, \infty)$, which is also a Borel set.
\end{enumerate}

This means that $\forall r \in \R \; f^{-1}((r, \infty))$ is a Borel set, hence $f$ is Borel.
\end{proof}

\begin{corollary} \label{C:MonotoneDecreasing}
Any monotone decreasing function $f: \R \mapsto \R$ is Borel.
\end{corollary}
\begin{proof}
Since $f$ is monotone decreasing, $-f$ is monotone increasing, which by Theorem \ref{T:monotone} is Borel. Therefore, by Lemma \ref{L:FlipBorel} $f$ is Borel.
\end{proof}

\begin{lemma} \label{L:FloorIsBorel}
Both the floor function $f(x) = \lfloor x \rfloor$ and the ceiling function $g(x) = \lceil x \rceil$ are Borel.
\end{lemma}
\begin{proof}
The both the floor function $f(x) = \lfloor x \rfloor$ and the ceiling function $g(x) = \lceil x \rceil$ are monotone increasing on $\R$, so by Theorem \ref{T:monotone} they are Borel.
\end{proof}

\subsection{Dimensionality-Reducing Borel Isomorphisms}

An extremely important example of a Borel isomorphism from $[0, 1]^{n}$ to $[0,1]$ is a map interchanging the digits of each value. This is useful for us because it can be used to lower the dimensionality of a data set. In fact, all dimensionality reducing Borel ismorphisms used in this project are based on this.

To understand this example of a Borel isomorphism (from $[0, 1]^2$ to $[0, 1]$), consider $x = (x_1, x_2) = (0.437, 0.982)$. Then (with base 10 in this example), $f(x) = 0.493872$, that is, the first digit from $x_1$ is taken, followed by the first digit from $x_2$, followed by the second digit of $x_1$, followed by the second digit of $x_2$, and so on. The special case of a number being exactly one can be handled by treating it as the infinite string 0.99999999.... (or an infinite string of the largest number in the base). This technique can be used with any arbitrary base and from $[0, 1]^{n}$ to $[0, 1]^{m}$ (where $m < n$) in general.


We now need to show this is a Borel isomorphism. We start by precisely defining this Borel isomorphism.

\begin{definition} \label{D:DigitInterlace}
Let $f: [0, 1]^n \mapsto [0, 1]$ be defined as follows, where $x_{j, i}$ is the $i^{th}$ digit of the $j^{th}$ variable (if $x_{j,i} = 1$, then we use $x_{j,i} = 0.bbb....$) and $b$ is the base (an integer greater than or equal to 2):
\begin{equation*}
f(x_1, x_2, ..., x_{n}) = \sum_{i=1}^{\infty} \sum_{j=1}^{n} b^{-(n \cdot (i - 1) + j)} x_{j, i}
\end{equation*}

Equivalently, we can define this function $f: [0, 1]^n \mapsto [0, 1]$ as follows, where $x_j$ is the $j^{th}$ variable and b is the base:
\begin{equation*}
f(x_1, x_2, ..., x_{n}) = \sum_{i=1}^{\infty} \sum_{j=1}^{n} b^{-(n \cdot (i - 1) + j)} \lfloor b^{i} x_{j} - b \lfloor b^{i-1} x_{j} \rfloor \rfloor
\end{equation*}
\end{definition}

We now need to show that this is in fact a Borel isomorphism.

\begin{theorem} \label{T:interchange} 
Let $f: [0, 1]^n \mapsto [0, 1]$ be the function defined in \ref{D:DigitInterlace}. Then $f$ is a Borel isomorphism.
\end{theorem}
\begin{proof}
We first flip the order of summation in $f$ as follows:
\begin{equation*}
f(x_1, x_2, ..., x_{n}) = \sum_{j=1}^{n} \sum_{i=1}^{\infty} b^{-(n \cdot (i - 1) + j)} \lfloor b^{i} x_{j} - b \lfloor b^{i-1} x_{j} \rfloor \rfloor
\end{equation*}

Factoring $b^{j}$ out, we obtain,
\begin{equation*}
f(x_1, x_2, ..., x_{n}) = \sum_{j=1}^{n} b^{-j} \sum_{i=1}^{\infty} b^{-n \cdot (i - 1)} \lfloor b^{i} x_{j} - b \lfloor b^{i-1} x_{j} \rfloor \rfloor
\end{equation*}

We first prove that $f: [0, 1]^{n} \mapsto [0, 1]$ is a well defined map, in particular that the infinite series converges and that the range of $f$ is $[0, 1]$.

We first notice that
\begin{align*}
\lfloor b^{i-1} x \rfloor \leq b^{i-1} x < \lfloor b^{i-1} x + 1 \rfloor
\end{align*}
This means that
\begin{align*}
0 \leq b^{i-1} x - \lfloor b^{i-1} x \rfloor < \lfloor b^{i-1} x + 1 \rfloor - \lfloor b^{i-1} x \rfloor = 1
\end{align*}
Multiplying by $b$, we find that
\begin{align*}
0 \leq b^{i} x - b \lfloor b^{i-1} x \rfloor < b
\end{align*}
Taking the floor, we find
\begin{align*}
0 \leq \lfloor b^{i} x - b \lfloor b^{i-1} x \rfloor \rfloor \leq b - 1
\end{align*}

This means that for each term $b^{-n \cdot (i - 1)} \lfloor b^{i} x_{j} - b \lfloor b^{i-1} x_{j} \rfloor \rfloor$ in the inner sum,
\begin{align*}
0 \leq b^{-n \cdot (i - 1)} \lfloor b^{i} x - b \lfloor b^{i-1} x \rfloor \rfloor \leq b^{-n \cdot (i - 1)} (b - 1)
\end{align*}

Since all terms are nonnegative, this means that if the series converges for the largest possible value for each term, it will converge for any possible values (and will converge absolutely, since each term is equal to its absolute value). It is obvious that the minimum possible value for the sum of the series is zero, by taking each term to be zero (the smallest possible value for each term). Taking the largest possible value for each term in the sum, we find that
\begin{align*}
f(x_1, x_2, ..., x_{n}) &\leq \sum_{j=1}^{n} b^{-j} \sum_{i=1}^{\infty} b^{-n \cdot (i - 1)} (b - 1) \\
&= b^{n} (b - 1) \sum_{j=1}^{n} b^{-j} \sum_{i=1}^{\infty} b^{-n \cdot i} \\
&= b^{n} (b - 1) \sum_{j=1}^{n} b^{-j} \left( \frac{1}{1-b^{-n}} - 1 \right) \\
&= b^{n} (b - 1) \left( \frac{b^{-n}}{1-b^{-n}} \right) \sum_{j=1}^{n} b^{-j} \\
&= b^{n} (b - 1) \left( \frac{b^{-n}}{1-b^{-n}} \right) \frac{1 - b^{-n}}{b-1} \\
&= 1
\end{align*}

This means that for all possible values of $x_1$, $x_2$, ..., $x_{n}$ (in $[0, 1]$), the series converges (by the comparison test) and furthermore $0 \leq f(x_1, x_2, ..., x_{n}) \leq 1$.

We must prove that $f$ is a bijection (onto $[0, 1]$).

To prove surjectivity, we notice that for any $y \in [0, 1]$, $y = 0. y_1 y_2 y_3 y_4 ...$, so we set $x_1 = 0. y_1 y_{n+1} y_{2n+1} ...$, $x_2 = 0. y_2 y_{n+2} y_{2n+2} ...$, ..., and $x_{n} = 0. y_{n} y_{2n} y_{3n} ...$, in which case $f(x_1, x_2, ..., x_{n}) = 0. y_1 y_2 ... y_{n} y_{n+1} ... y_{2n} y_{2n+1} ...$, and so $f$ is surjective onto $[0, 1]$.

To prove that $f$ is injective, suppose that $(x_1, x_2, ..., x_{n}) \neq (x_1', x_2', ..., x_{n}')$. Then for some $1 \leq i \leq n$, $x_{i} \neq x_{i}'$, which means that $x_{i}$ and $x_{i}'$ differ in some digit $k$. Then let $y = f(x_1, x_2, ..., x_{n})$ and $y' = f(x_1', x_2', ..., x_{n}')$, which means that $y_{n k + i} = x_{i, k}$ (where $y_{n k + i}$ is the ${n k + i}^{th}$ digit of $y$ and $x_{i, k}$ is the $k^{th}$ digit of $x_{i}$) and $y_{n k + i}' = x_{i, k}'$. Therefore $y_{n k + i} \neq y_{n k + i}'$ and hence $y \neq y'$, which means that $f$ is injective.

Next we show that $f$ is Borel.

We first show that $g_{j}$ is Borel (for $1 \leq j \leq n$), where $g_{j}$ is defined by
\begin{align*}
g_{j}(x_{j}) = \sum_{i=1}^{\infty} b^{-n \cdot (i - 1)} \lfloor b^{i} x_{j} - b \lfloor b^{i-1} x_{j} \rfloor \rfloor
\end{align*}

We notice that $g_{j}$ is the pointwise limit of a series of functions
\begin{align*}
g_{j}(x_{j}) = \lim_{k \to \infty} \sum_{i=1}^{k} b^{-n \cdot (i - 1)} \lfloor b^{i} x_{j} - b \lfloor b^{i-1} x_{j} \rfloor \rfloor
\end{align*}

From Lemma \ref{L:FloorIsBorel} and Theorem \ref{T:BorelLinearCombination}, it is obvious that $b^{-n \cdot (i - 1)} \lfloor b^{i} x_{j} - b \lfloor b^{i-1} x_{j} \rfloor \rfloor$ is Borel, and so the sum of a finite number of such terms is finite (by Theorem \ref{T:BorelSum}), hence for all $k \geq 1$, $\sum_{i=1}^{k} b^{-n \cdot (i - 1)} \lfloor b^{i} x_{j} - b \lfloor b^{i-1} x_{j} \rfloor \rfloor$ is Borel. Since this series converges pointwise as $k \to \infty$ to $g_{j}$, this means that (by Theorem \ref{T:PointwiseBorel}) $g_{j}$ is Borel.

We notice that
\begin{align*}
f(x_1, x_2, ..., x_{n}) = \sum_{j=1}^{n} b^{-j} g_{j}(x_{j})
\end{align*}

This is the linear combination of a finite number of Borel functions, which by Theorem \ref{T:BorelLinearCombination} is Borel.

Since $f$ is a Borel bijective mapping between two complete separable metric spaces, by Theorem \ref{T:BorelBijectionIsIsomorphism} $f$ is a Borel isomorphism.
\end{proof}

Similarly, we can use this to define a Borel isomorphism $g: [0, 1]^{n} \mapsto [0, 1]^{m}$, where $m < n$. We do this by deciding which dimensions will be combined, and apply the above Borel isomorphism $f$ to map each set of these to $[0, 1]$. For instance, if we want to do an 6 to 3 dimensional Borel isomorphic reduction (from $[0, 1]^6$ to $[0, 1]^3$), we apply $f$ to combine dimensions 1 and 2 into one, dimensions 3 and 4 into one, and dimensions 5 and 6 into one.

\subsection{Normalizing the Data}

One problem is that elements in the original data may not necessarily be in $[0, 1]$ and could take on values outside that interval. To fix this we force the data to the interval $[0, 1]$ by applying the function $h:[x_{min}, x_{max}] \mapsto [0, 1]$ to each element of the data set (where $x_{min}$ is the minimum value in the data set and $x_{max}$ is the maximum value in the data set):
\begin{equation}
h(x) = \frac{x - x_{min}}{x_{max} - x_{min}}
\end{equation}

For any value $x \in [x_{min}, x_{max}]$, we have that
\begin{equation}
h(x) = \frac{x - x_{min}}{x_{max} - x_{min}} \leq \frac{x_{max} - x_{min}}{x_{max} - x_{min}}  = 1
\end{equation}
We also have that
\begin{equation}
h(x) = \frac{x - x_{min}}{x_{max} - x_{min}} \geq \frac{x_{min} - x_{min}}{x_{max} - x_{min}}  = 0
\end{equation}
Combining these results, we find that
\begin{equation}
0 \leq h(x) \leq 1
\end{equation}
This means that $\forall x \in [x_{min}, x_{max}]$, $h(x) \in [0, 1]$. This means that $h$ is a map to $[0, 1]$.

We also need to show that $h:[x_{min}, x_{max}] \mapsto [0, 1]$ is a Borel isomorphism.
\begin{theorem} \label{T:normalize}
Let $h:[x_{min}, x_{max}] \mapsto [0, 1]$ be defined as
\begin{equation*}
h(x) = \frac{x - x_{min}}{x_{max} - x_{min}}
\end{equation*}
Then $h$ is a homeomorphism and therefore a Borel isomorphism.
\end{theorem}
\begin{proof}
We first show that $h$ has the inverse $h^{-1}: [0, 1] \mapsto [x_{min}, x_{max}]$ where
\begin{equation}
h^{-1}(x) = x(x_{max} - x_{min}) + x_{min}
\end{equation}

To show this we first show that $h(h^{-1}(x)) = x$,
\begin{align*}
h(h^{-1}(x)) &= h(x(x_{max} - x_{min}) + x_{min}) \\
&= \frac{x(x_{max} - x_{min}) + x_{min} - x_{min}}{x_{max} - x_{min}} \\
&= \frac{x(x_{max} - x_{min})}{x_{max} - x_{min}} \\
&= x
\end{align*}

We also need to show that $h^{-1}(h(x)) = x$,
\begin{align*}
h^{-1}(h(x)) &= h^{-1}\left(\frac{x - x_{min}}{x_{max} - x_{min}}\right) \\
&= \frac{x - x_{min}}{x_{max} - x_{min}}(x_{max} - x_{min}) + x_{min} \\
&= x - x_{min} + x_{min} \\
&= x
\end{align*}

Since $h(h^{-1}(x)) = h^{-1}(h(x)) = x$ for all $x$ in the domain, this means that $h$ has the inverse $h^{-1}$ and so is bijective.

Since both $h$ and $h^{-1}$ are linear functions, both $h$ and $h^{-1}$ are continuous, which means that $h$ is a bijective continuous function with a continuous inverse, or a homeomorphism. Since $h$ is a homeomorphism, by Theorem \ref{T:homeomorphism} $h$ is also a Borel isomorphism.
\end{proof}

Since the function $h:[x_{min}, x_{max}] \mapsto [0, 1]$ defined above used to force data to $[0, 1]$ is a Borel isomorphism and the composition of Borel isomorphisms is a Borel isomorphism (by Theorem \ref{T:composition}), we can apply $h$ to the data before applying the digit interchange Borel isomorphism and the result is still a Borel isomorphism.

\subsection{Borel Dimensionality Reduction of Datasets}

The reason we are interested in Borel isomorphisms is that if we apply a Borel isomorphism from a separable metric space $X$ to a metric space where k-NN is universally consistent (such as $\R^n$ or any metric subspace of $\R^n$), then the k-NN learning rule is universally consistent on $X$ with the Borel isomorphism applied. More formally ~\cite{Pestov}
\begin{theorem}
Let $X$ be a separable metric space and $Y$ a metric space (with metric $d$) such that k-NN is universally  consistent, $f: X \mapsto Y$ be a Borel isomorphism, and $\rho$ be a metric on $X$ defined by
\begin{equation*}
\rho(x, y) = d(f(x), f(y))
\end{equation*}

Then the k-NN learning rule on $X$ with the metric $\rho$ is universally consistent.
\end{theorem}

This result (proven in \cite{Pestov}) means that in the limit, as the number of sample points goes to infinity, the error of the k-NN learning rule on any separable metric space converges to the Bayes error. This does not mean that for any particular sample size, the error after applying the Borel isomorphism will remain the same. For a particular number  of sample points, it is possible that applying a Borel isomorphism may increase the error, but in the limit as the number of sample points goes to infinity, k-NN after applying the Borel isomorphism will converge to the Bayes error. In this project, we will see if dimensionality reducing Borel isomorphisms preserve the accuracy well and look at methods for increasing their accuracy. There are multiple approaches we will try, including adjusting the base and applying a linear transformation before performing the digit interlace Borel isomorphism.

\section{Datasets}

We used various datasets in our project, including the Yeast dataset and the Phoneme dataset. We use the kNN classifier with the Euclidean distance.

The yeast dataset is available at the UCI machine learning repository. \cite{UCI} \cite{yeast} In the yeast dataset, we have a total of 10 columns and 1484 entries. Of the columns, one is a string label indicating the sample, 8 are real values with data about the yeast sample, and the last is the category, which can take on 10 possible values. This means that the Yeast dataset has 8 dimensions. The classification accuracy is not very good, being approximately 57\% on average with kNN applied directly, with no dimensionality reduction (this is still far better than random choice, as the most frequent outcome occurs approximately 30 \% of the time, so the accuracy of random choices is expected to be no higher than approximately 30 \%). The Yeast dataset is not very high dimensional, but is good for testing purposes as simulations run quickly on the Yeast dataset.

Another dataset is the phoneme dataset.\cite{phoneme} This dataset consists of various phonemes being spoken, with the data being the waveform data and the attributes being the phonemes spoken. This dataset has medium high dimensionality with 256 dimensions, so it is good as an actual dataset for practice. The phoneme dataset has 4508 entries and there are 5 possible classes for the response (corresponding to the different 5 phonemes being spoken). There is also an additional column identifying the speaker.

\section{Optimal k}

Whenever we use the kNN classifier, we need to find the optimal value of $k$ for our dataset. This is the value of $k$ that we use for that test so we may have the highest accuracy. In general, the value of $k$ changes if we apply a transformation to the data (such as the digit interlace Borel isomorphism), so we should find the new optimal value of $k$ for the dataset with that transformation applied. To find the optimal $k$, we split the dataset into the training and testing sets, and find the value of $k$ such that kNN maximizes the accuracy on the testing set.

As an example, we find the optimal value for $k$ for the Yeast dataset (with an 8 to 1 Borel isomorphic reduction performed). To do this we test the accuracy with $k$ being integers between 1 and 30 (with 500 trials, with the dataset randomly split into training and testing sets, and an 8 to 1 reduction performed), and we obtain the box-and-whiskers plot in figure \ref{F:YeastK}.
\begin{figure}[h!]
\includegraphics[width=\textwidth]{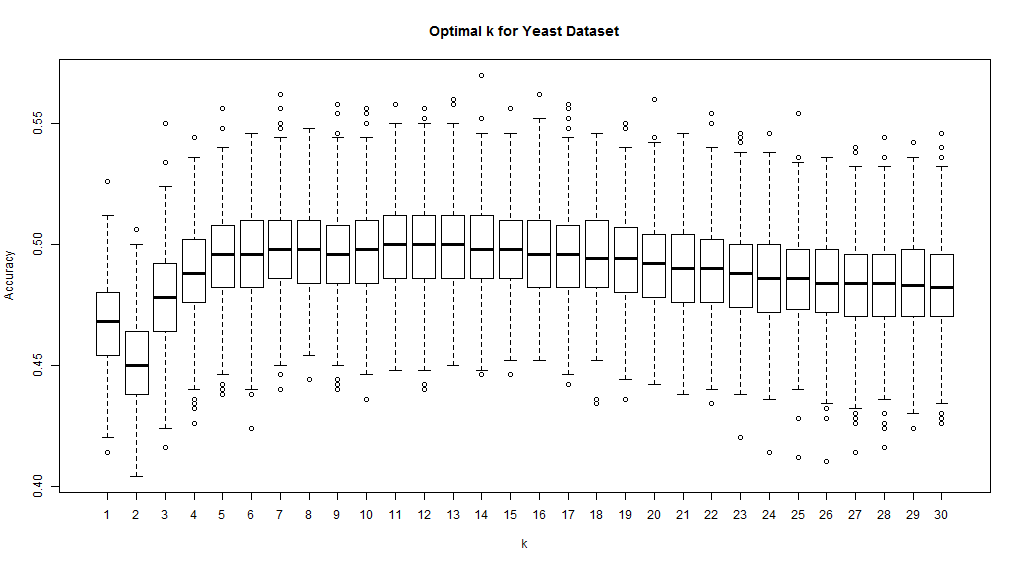} 
\caption{Accuracy of Borel isomorphic reduction with various $k$ for the Yeast dataset.}
\label{F:YeastK}
\end{figure}

From the values of k, the maximum is at $k = 11$. We notice that there is no significant difference between $k=11$ and $k=13$, for instance, but this is not important for us, we can simply select any k that appears to be optimal. We will therefore use $k=11$ here.

For the phoneme dataset in its original form we find that $k = 21$ is optimal. After the Borel isomorphism is applied we need to find the new optimal value of $k$ to use in our tests.

\section{Optimal Base}

One interesting problem is to determine the optimal base for the Borel isomorphism. Smaller bases (like 2) would result in finer mixing of the values in each column (that is, columns after the first one have a larger effect), while larger bases would result in a coarser mixing and will increase the effect of the first columns and reduce the effect of later columns. We will therefore try various bases on the datasets and see which produce the most accurate results. This is one way of expanding the forms of possible Borel isomorphisms (we can try multiple bases instead of only one base), and we can then pick the one with the highest accuracy. The importance of this work is briefly discussed further in the conclusion.

We first do a test of this with the Yeast dataset. To check this we selected bases between 2 and 20 and did a test of the accuracy of reducing 8 dimensions to 1 with that base. The result is summarized in the box-and-whiskers plot in figure \ref{F:YeastBase}.

\begin{figure}[h!]
\includegraphics[width=\textwidth]{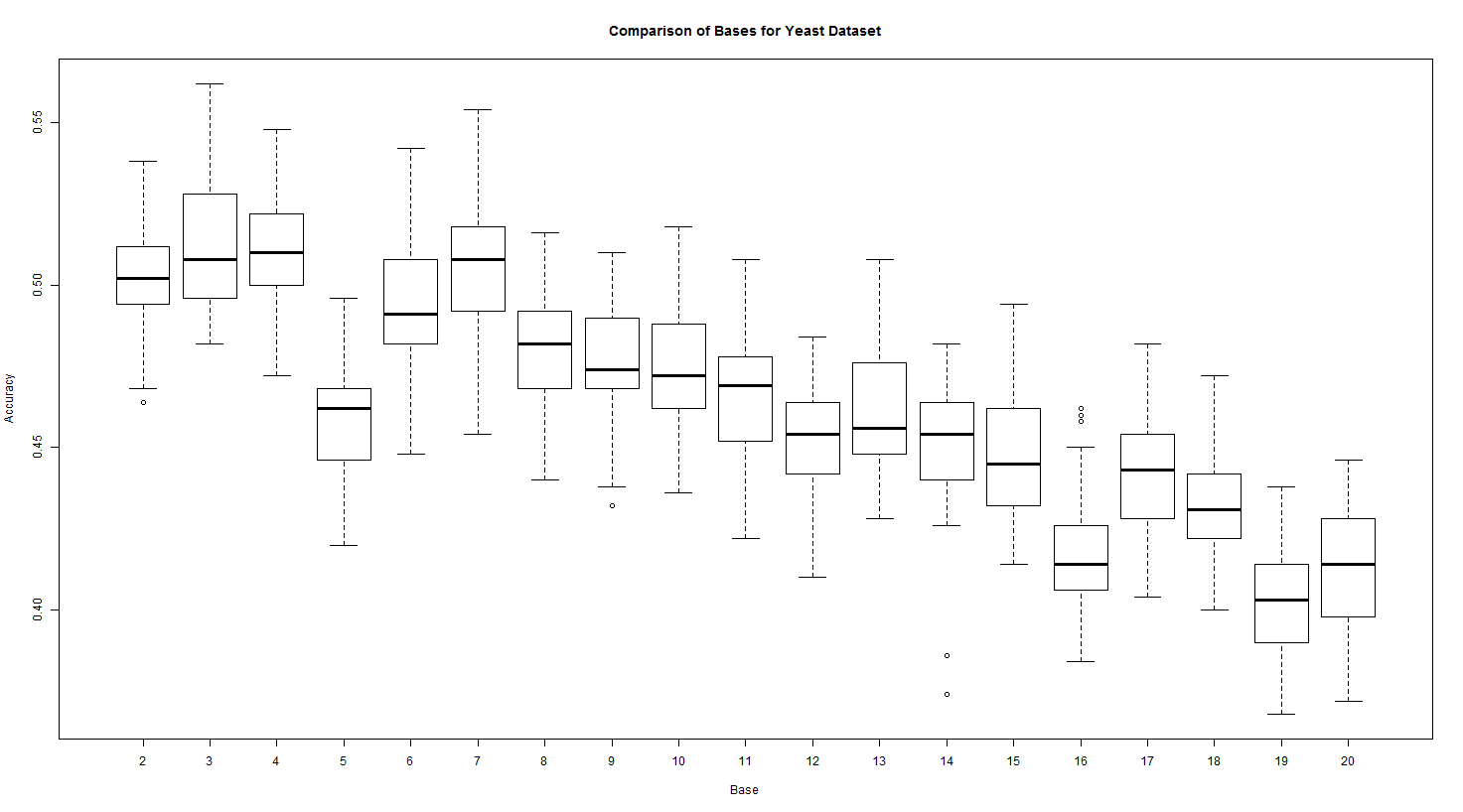} 
\caption{Comparison of various bases for the Yeast dataset, with an 8 to 1 dimensionality reduction.}
\label{F:YeastBase}
\end{figure}

We see from this plot that bases 3, 4, and 7 appear to be optimal for the Yeast dataset.

For the Phoneme dataset, the results for the accuracy of the base depend enormously on how many dimensions are compressed into one. If all 256 dimensions are compressed into one (for a 256 to 1 dimension reduction), then base 3 is the optimal base as expected, as seen in figure \ref{F:PhonemeBase16}. However, if only 16 dimensions are compressed into one (for a 256 to 16 dimension reduction), then base 3 is suboptimal and larger bases perform better (with 14 having the highest performance), as seen in figure \ref{F:PhonemeBase256}. The optimal base for 16, 32, 64, 128, and 256 dimensions compressed into one is summarized in table \ref{F:PhonemeBaseTable}.

\begin{figure}[h!]
\includegraphics[width=\textwidth]{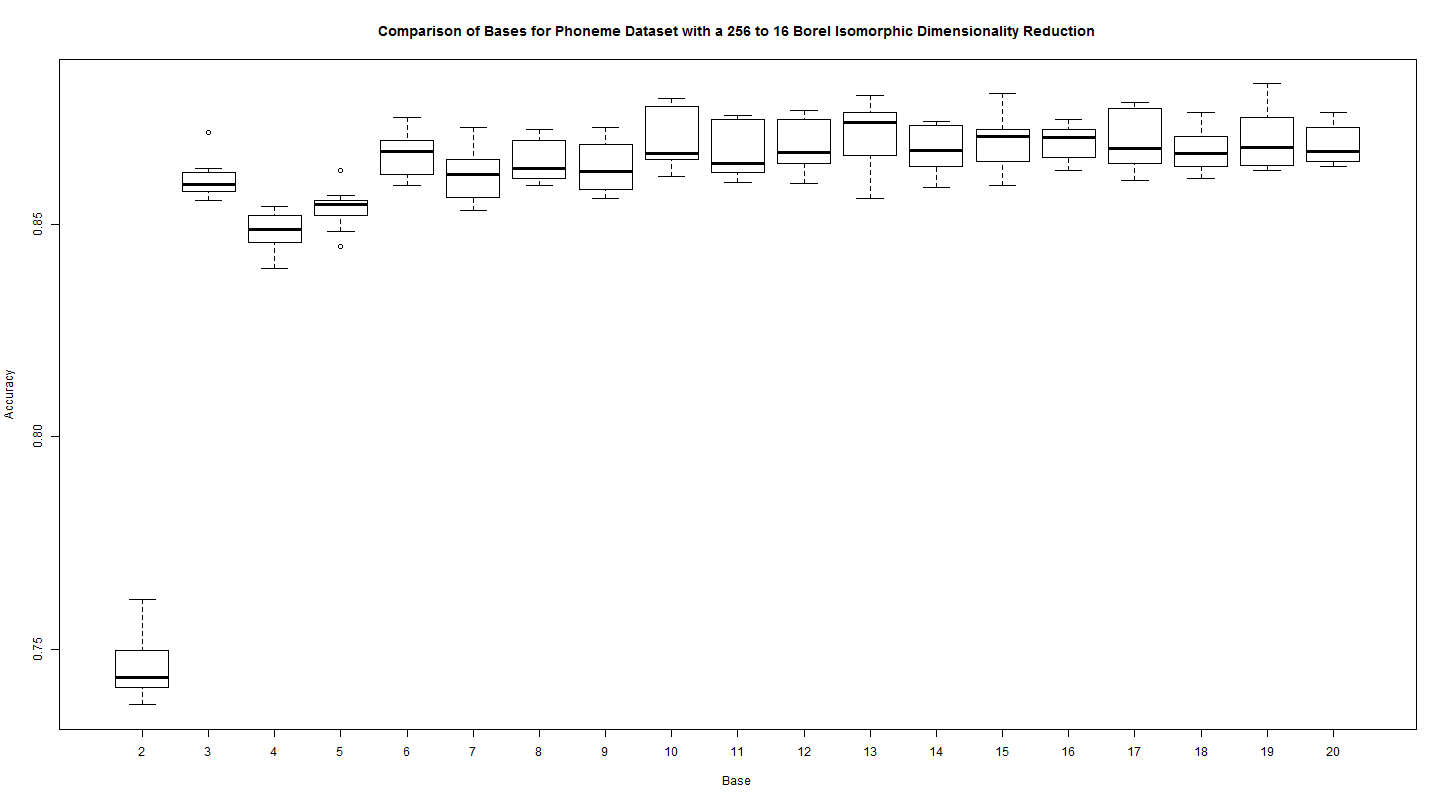}
\caption{Comparison of various bases with a 256 to 16 dimension reduction for the Phoneme dataset.}
\label{F:PhonemeBase16}
\end{figure}

\begin{figure}[H]
\includegraphics[width=\textwidth]{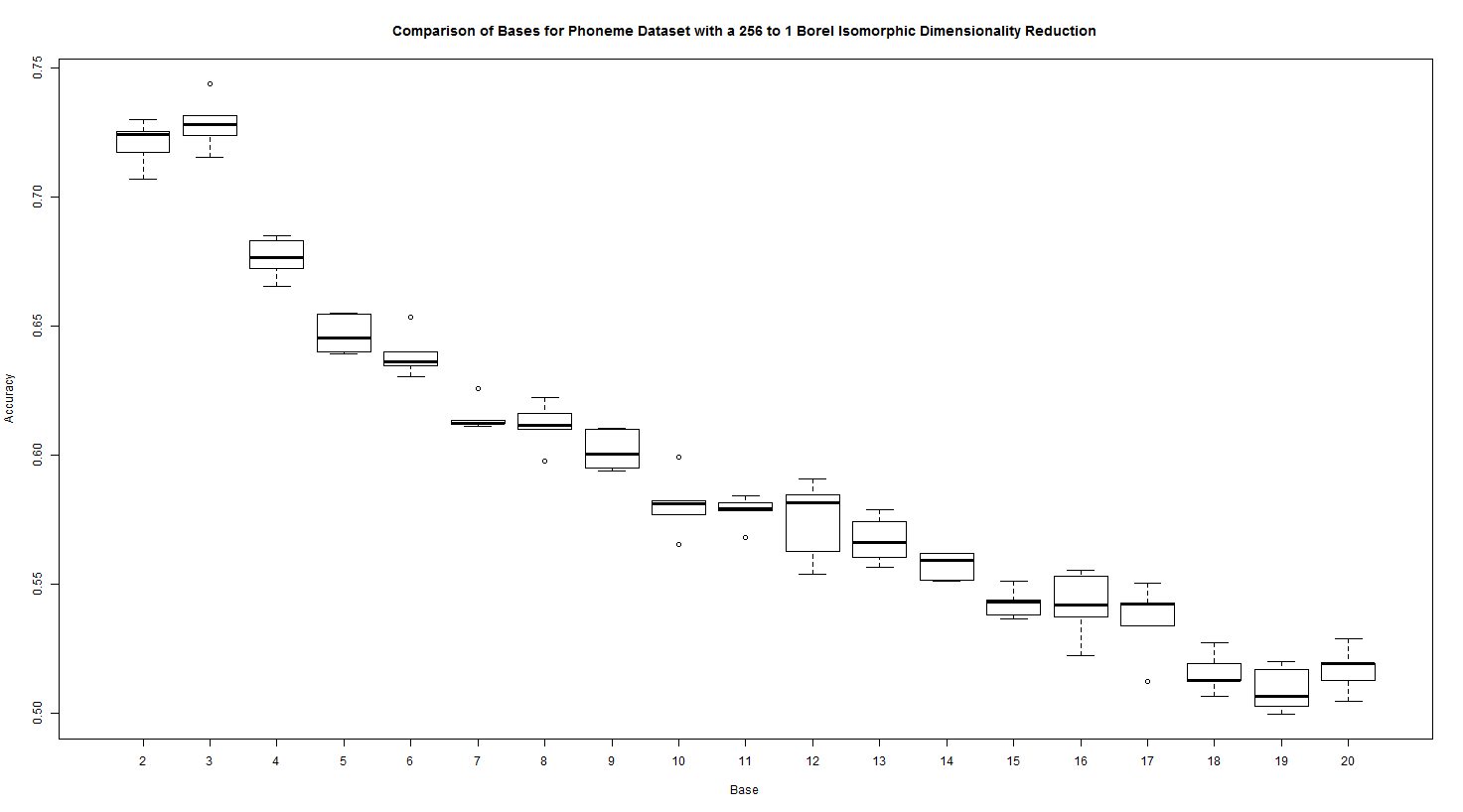} 
\caption{Comparison of various bases with a 256 to 1 dimension reduction for the Phoneme dataset.}
\label{F:PhonemeBase256}
\end{figure}

\begin{table}[ht]
\centering
\begin{tabular}{|r|r|r|}
  \hline
 Dimensions Compressed into One & Dimensions after Reduction & Optimal base \\ 
   \hline
 16 & 256 to 16 & 13 \\
   \hline
 32 & 256 to 8 & 24 \\
   \hline
 64 & 256 to 4 & 3 \\
  \hline
 128 & 256 to 2 & 3 \\
 \hline
 256 & 256 to 1 & 3 \\
   \hline%
\end{tabular}
\caption{A table showing the optimal base for the Phoneme dataset with various numbers of dimensions compressed into one.}
\label{F:PhonemeBaseTable}
\end{table}

\section{Accuracy of the Borel Isomorphic Dimensionality Reduction}

We now consider the overall accuracy of classifying the data after applying the Borel isomorphism, compared to directly applying kNN without the Borel isomorphism.

For the Yeast dataset, we compare the accuracy of classifying the original dataset and reducing the number of dimensions from 8 to 4, 2, and 1. We then compare the accuracy of the Borel isomorphic dimensionality reduction for various numbers of dimensions compressed into a single dimension. We compare 8 to 1, 4 to 1, 2 to 1, and no dimension reduction against each other. We select the optimal base and $k$ for each dimensionality reduction. We obtain the box-and-whiskers plot in figure \ref{F:YeastResult}.
\begin{figure}[H] \label{F:YeastResult}
\includegraphics[width=\textwidth]{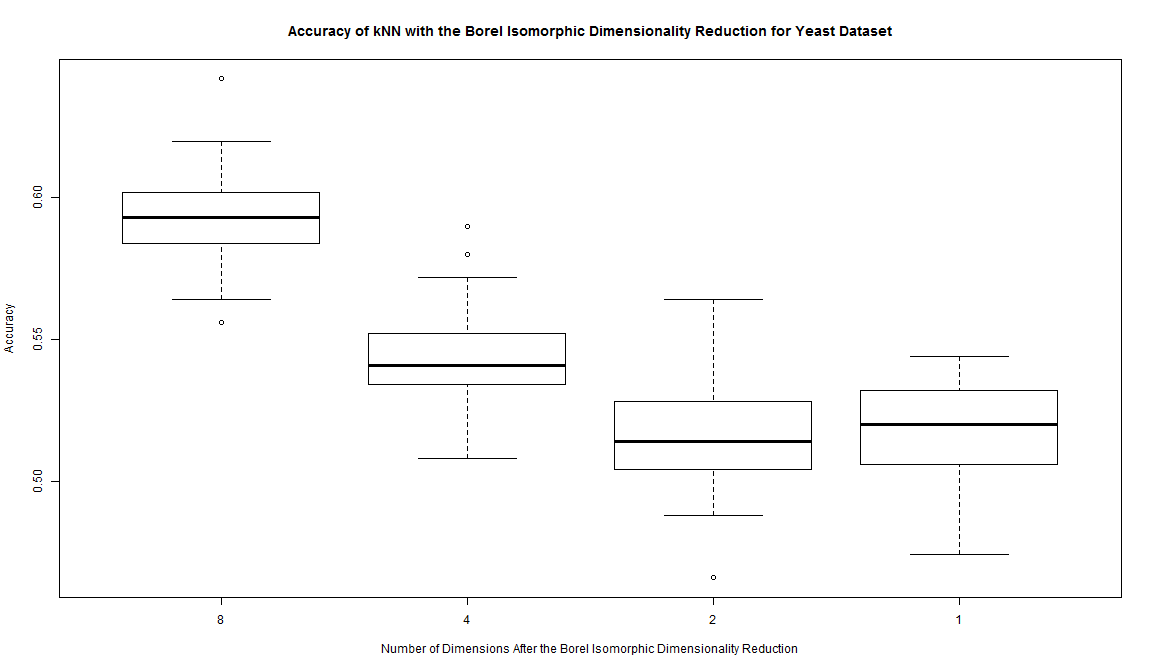} 
\caption{Accuracy of Borel Isomorphic Dimension Reduction for Yeast Dataset}
\end{figure}

We find that the accuracy decreases from 57\% with no Borel isomorphism to 50\% with an eight to one Borel isomorphic reduction. This means that Borel isomorphic dimensionality reduction works well for the Yeast dataset.

A box and whiskers plot of the accuracy of the Borel isomorphism applied to the Phoneme dataset (with the optimal base and $k$ for each dimensionality reduction) is in figure \ref{F:PhonemeResult}.
\begin{figure}[H]
\includegraphics[width=\textwidth]{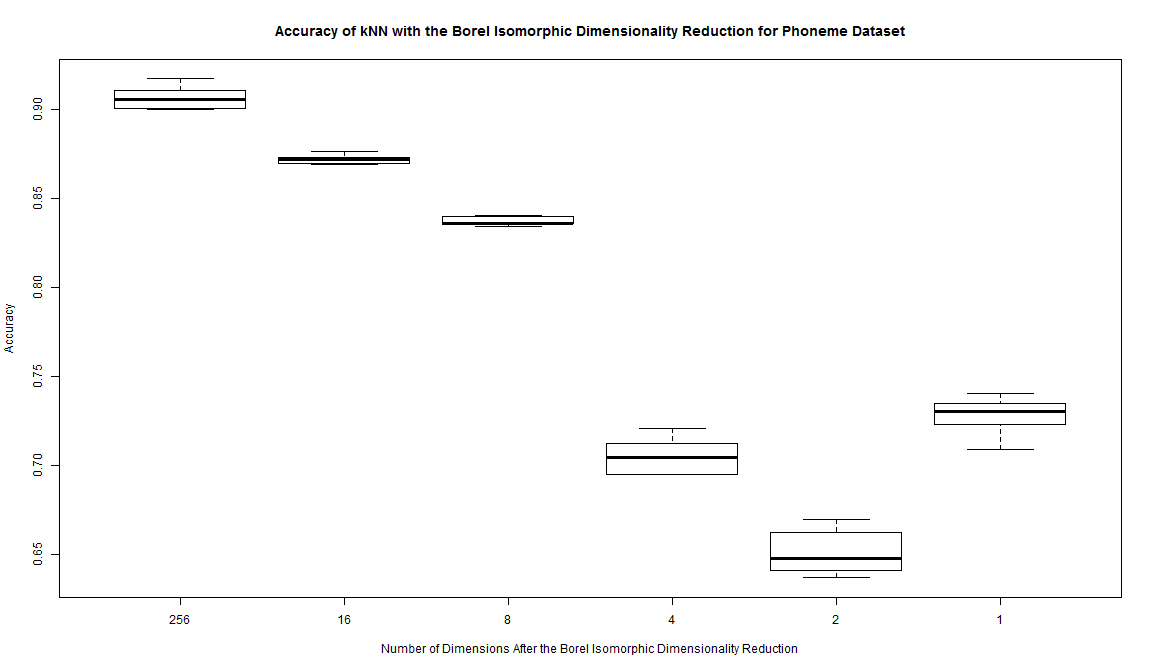}
\caption{Accuracy of Borel Isomorphic Dimensionality Reduction for Phoneme Dataset}
\label{F:PhonemeResult}
\end{figure}

The accuracy with no Borel isomorphic dimensionality reduction is approximately 90\%. We see that for a 16 dimensions compressed into one, the accuracy decreases very little to 87 \%, which is almost the original accuracy. For a Borel isomorphic reduction of the entire dataset to one dimension, the accuracy is approximately 77 \%.

We also compare the Borel isomorphic dimensionality reduction to two more traditional methods, Principal Component Analysis (PCA) and Linear Discriminant Analysis (LDA). Principal Component Analysis applies an orthogonal transformation that converts the variables to a set of linearly uncorrelated principal components such that the first principal component has the largest possible variance, the second principal component has the second largest possible variance, and so on. Linear Discriminant Analysis looks for linear combinations of features that separates the data points into two or more classes. For the Yeast dataset, both PCA and LDA perform worse than the Borel isomorphic dimensionality reduction. For the Phoneme dataset, both PCA and LDA perform better  than Borel isomorphic dimensionality reduction for smaller dimensionality reductions, but for reductions to very low dimensional spaces (such as 256 to 1 dimension) the Borel isomorphic dimensionality reduction performs better.

In this section, we simply applied the Borel isomorphism to the columns without doing any reordering of the columns. This may not be optimal, however. This is the idea for the next section.

\section{Multiplication by Orthogonal Matrices}

\subsection{Motivation}

In the initial dataset, the columns are in some order. When we apply the Borel isomorphism, we apply it to the columns in that same order, so that the first digit is from the first column, the second digit from the second column, and so on. This may not be optimal, however. For instance, it is possible that the first column may be a very weak predictor while a later column may be a strong predictor. One possible solution is to multiply the original data matrix by another invertible matrix, which (if we select the right matrix) may increase the accuracy. In this way, we enlarge the forms of available Borel isomorphisms, and we may choose the one with the highest accuracy. A further justification for enlarging the forms of possible Borel isomorphisms is a topic for future research.

\begin{theorem} \label{T:invmatrix}
Multiplying the data matrix (with $n$ columns) by any $n$ by $n$ invertible matrix and then applying any Borel isomorphism is a Borel isomorphism.
\end{theorem}
\begin{proof}
Multiplication by an invertible matrix is a bijective function (since multiplying by its inverse matrix is the inverse, is it bijective) and any linear map (like matrix multiplication) on a finite dimensional vector space is bijective~\cite{DiscontinousLinearMap}, and its inverse is also a finite dimensional linear map and is also continuous, which means that (by Theorem \ref{T:homeomorphism}), multiplication by an invertible matrix is a Borel isomorphism. Since composition of Borel isomorphisms is a Borel isomorphism (by Theorem \ref{T:composition}), the composition of multiplying by an invertible matrix and then applying any other Borel isomorphism is a Borel isomorphism.
\end{proof}

This means that we can first multiply the data matrix by any invertible matrix of our choice and then apply the digit interchange Borel isomorphism.

The problem now is to find a matrix that improves the accuracy. We would like a matrix that represents a linear transformation that rotates vectors in the data set, preferably to one that improves the accuracy. Most invertible matrices correspond to linear maps that are distortions, dilations, and contractions. We are not interested in such transformations. We are only interested in linear transformations that correspond to rotations. Such matrices form a group (under multiplication, result proven below) called the \emph{Orthogonal group}, denoted $O(n)$. Furthermore if we restrict ourselves to matrices that preserve orientation, we form the \emph{Special Orthogonal} group ($SO(n)$).

\subsection{The Orthogonal and Special Orthogonal Groups}

\begin{definition} \label{D:SO(n)}
An $n$ by $n$ matrix $R$ is an \emph{orthogonal} matrix if multiplying the matrix by its transpose (in either order) results in the identity matrix: \cite{FriedbergInselSpence} \begin{equation*}
R^\mathsf{T} R = R R^\mathsf{T} = I
\end{equation*}

We say that $R$ is \emph{special orthogonal} if it satisfies the additional condition that the determinant is one:
\begin{equation*}
det(R) = 1
\end{equation*}
\end{definition}

We denote the set of orthogonal $n$ by $n$ matrices $O(n)$ and the set of special orthogonal $n$ by $n$ matrices $SO(n)$. We now show the $O(n)$ and $SO(n)$ are groups under multiplication, with $SO(n)$ being a subgroup of $O(n)$.

\begin{theorem} \label{T:SO(n)IsGroup}
The set of $n$ by $n$ orthogonal matricies $O(n)$ froms a group under multiplication, and  the set of special orthogonal matrices $SO(n)$ is a subgroup of $O(n)$.
\end{theorem}
\begin{proof}
We first show $O(n)$ satisfies the definition of a group:

\begin{enumerate}
\item Show $O(n)$ is closed under multiplication.

Let $Q, R \in O(n)$, we must show that $QR \in O(n)$.

We have that
\begin{equation*}
(QR)^\mathsf{T} (QR) = R^\mathsf{T} Q^\mathsf{T} Q R = R^\mathsf{T} I R = R^\mathsf{T} R = I
\end{equation*}

We also find that
\begin{equation*}
(QR) (QR)^\mathsf{T} = Q R R^\mathsf{T} Q^\mathsf{T} = Q I Q^\mathsf{T} = Q Q^\mathsf{T} = I
\end{equation*}

This means that
\begin{equation*}
(QR)^\mathsf{T} (QR) = (QR) (QR)^\mathsf{T} = I
\end{equation*}

Therefore, since $(QR)^\mathsf{T} (QR) = (QR) (QR)^\mathsf{T} = I$, this means $QR \in O(n)$.

\item Show that $I \in O(n)$.

Since $I^\mathsf{T} I = I I^\mathsf{T} = I$, $I \in O(n)$.

\item Since matrix multiplication is associative, it follows that multiplication of matrices in $O(n)$ is associative.

\item Show every element in $O(n)$ has an inverse in $O(n)$, by proving that if $Q \in O(n)$, then $Q^{-1} \in O(n)$.

Since $Q^\mathsf{T} Q = Q Q^\mathsf{T} = I$, by the definition of a matrix inverse $Q^\mathsf{T}$ is the inverse of $Q$, so $Q^{-1} = Q^\mathsf{T}$.

We see that
\begin{align*}
\left(Q^{-1}\right)^\mathsf{T} Q^{-1} &= \left(Q^\mathsf{T}\right)^\mathsf{T} Q^\mathsf{T} \\
									  &= Q Q^\mathsf{T}\\
									  &= I
\end{align*}

Similarly, we see that
\begin{align*}
Q^{-1} \left(Q^{-1}\right)^\mathsf{T} &= Q^\mathsf{T} \left(Q^\mathsf{T}\right)^\mathsf{T} \\
									  &= Q^\mathsf{T} Q \\
									  &= I
\end{align*}

Since $\left(Q^{-1}\right)^\mathsf{T} Q^{-1} = Q^{-1} \left(Q^{-1}\right)^\mathsf{T} = I$, this means that $Q^{-1} \in O(n)$.

\end{enumerate}

Next we must show that $SO(n)$ is a subgroup of $O(n)$. Since a matrix in $SO(n)$ must satisfy all the conditions a matrix in $O(n)$ does, it is obvious that $SO(n)$ is a subset of $O(n)$. We use the subgroup test to show $SO(n)$ is a subgroup of $O(n)$. \cite{Nicholson}

\begin{enumerate}
\item Show $I \in SO(n)$.

Since $I^\mathsf{T} I = I I^\mathsf{T} = I$ and $det(I) = 1$, $I \in SO(n)$.

\item Show if $Q, R \in SO(n)$, then $QR \in SO(n)$.

We see that
\begin{equation*}
det(QR) = det(Q) det(R) = 1 \cdot 1 = 1
\end{equation*}

Since $QR \in O(n)$ (shown above) and $det(QR) = 1$, it follows that $QR \in SO(n)$.

\item Show if $Q \in SO(n)$, then $Q^{-1} \in SO(n)$.

We have already shown that since $Q \in O(n)$, $Q^{-1} \in O(n)$.

Since $det(Q) = 1$, we find that
\begin{align*}
det\left(Q^{-1}\right) = \frac{1}{det(Q)} = \frac{1}{1} = 1
\end{align*}

Combining these facts, we find that $Q^{-1} \in O(n)$ and $det\left(Q^{-1}\right) = 1$, which means $Q^{-1} \in SO(n)$.

\end{enumerate}

\end{proof}

\begin{theorem} \label{T:OrthogonalDeterminant}
For any matrix $Q \in O(n)$, $det(Q) = \pm 1$.
\end{theorem}
\begin{proof}
Let $Q \in O(n)$. Since $Q^{\mathsf{T}} Q = I$, it follows that $det(Q^{\mathsf{T}} Q) = det(I) = 1$. We find that
\begin{align*}
det(Q^{\mathsf{T}} Q) &= det(Q^{\mathsf{T}}) det(Q) \\
&= det(Q) det(Q) \\
&= \left( det(Q) \right)^2
\end{align*}
Since $det(Q^{\mathsf{T}} Q) = 1$, it follows that $\left( det(Q) \right)^2 = 1$, and so $det(Q) = \pm 1$.
\end{proof}

\begin{corollary} \label{C:OnSquaresIsSOn}
If $Q \in O(n)$, then $Q^2 \in SO(n)$.
\end{corollary}
\begin{proof}
Let $Q \in O(n)$, then by Theorem \ref{T:OrthogonalDeterminant}, $det(Q) = \pm 1$. If $det(Q) = 1$, then $det(Q^2) = \left(det(Q)\right)^2 = 1$. If $det(Q) = -1$, then $det(Q^2) = \left(det(Q)\right)^2 = \left(-1\right)^2 = 1$. Since $O(n)$ is a group (by Theorem \ref{T:SO(n)IsGroup}) and $Q \in O(n)$, $Q^2 \in O(n)$, and since $det(Q^2) = 1$, $Q^2 \in SO(n)$.
\end{proof}

\begin{corollary} \label{C:OnMultIsBorelIsomorphism}
Let $Q \in O(n)$ be an orthogonal matrix in $O(n)$. Then $f: \R^{n} \mapsto \R^{n}$ where $f(\vec x) = Q \vec x$ is a Borel isomorphism.
\end{corollary}
\begin{proof}
By Theorem \ref{T:OrthogonalDeterminant}, for any matrix $Q \in O(n)$, $det(Q) = \pm 1$, which means that $det(Q) \neq 0$. A matrix is singular if and only if its determinant is zero, which means that $Q$ is invertible. Since multiplication by any invertible matrix is a Borel isomorphism (by Theorem \ref{T:invmatrix}), multiplication by $Q$ is a Borel isomorphism.
\end{proof}

One important property of $O(n)$ is that the columns are orthogonal to each other.

\begin{theorem} \label{T:ColumnsOrthogonal}
An $n$ by $n$ matrix $Q$ is in $O(n)$ if and only if the columns of $Q$ and the rows of $Q$ are orthonormal to each other.

That is, $Q \in O(n)$ if and only if $\forall\; 1 \leq i, j \leq n$ where $i \neq j$ and $v_{i}, v_{j}$ are both column vectors or both row vectors of $Q$, $\vec v_i \cdot \vec v_i = 1$ and $\vec v_i \cdot \vec v_j = 0$.
\end{theorem}
\begin{proof}
First we show that if $Q \in O(n)$, then the rows of $Q$ are orthonormal and the columns of $Q$ are orthonormal. Let $Q \in O(n)$.

The the column vectors of $Q$ be $\vec v_1$, $\vec v_2$, ..., $\vec v_{n}$, so that $Q = \begin{bmatrix}\vec v_1& \vec v_2& ...& \vec v_{n}\end{bmatrix}$.

Then we find that
\begin{align*}
I &= Q^\mathsf{T} Q \\
&= \begin{bmatrix}
\vec v_1^\mathsf{T} & \\
\vec v_2^\mathsf{T} & \\
... & \\
\vec v_{n}^\mathsf{T} & \\
\end{bmatrix} \begin{bmatrix}\vec v_1& \vec v_2& ...& \vec v_{n}\end{bmatrix}
\end{align*}

This means that (for all $i \in \{0, 1, ..., n\}$) $\vec v_{i} \cdot \vec v_{i} = \vec v_{i}^\mathsf{T} \vec v_{i} = 1$, and that (for $i \neq j$) $\vec v_{i} \cdot \vec v_{j} = \vec v_{i}^\mathsf{T} \vec v_{j} = 0$, so the columns of $Q$ form an orthonormal set.

Now let the row vectors of $Q$ be $\vec u_1$, $\vec u_2$, ..., $\vec u_{n}$, so that $Q = \begin{bmatrix}
\vec u_1&\\
\vec u_2&\\
...&\\
\vec u_{n}
\end{bmatrix}$.

We then find that
\begin{align*}
I &= Q Q^\mathsf{T} \\
&= \begin{bmatrix}
\vec u_1&\\
\vec u_2&\\
...&\\
\vec u_{n}
\end{bmatrix} \begin{bmatrix}\vec u_1^\mathsf{T}& \vec u_2^\mathsf{T}& ...& \vec u_{n}^\mathsf{T} \end{bmatrix}
\end{align*}

This means that (for all $i \in \{0, 1, ..., n\}$) $\vec u_{i} \cdot \vec u_{i} = \vec u_{i} \vec u_{i}^\mathsf{T} = 1$, and that (for $i \neq j$) $\vec u_{i} \cdot \vec u_{j} = \vec u_{i} \vec u_{j}^\mathsf{T} = 0$, so the rows of $Q$ form an orthonormal set.

This means that if $Q \in O(n)$, then the rows and the columns of $Q$ both form orthonormal sets.

We now need to prove the opposite implication. Let $Q$ be an $n$ by $n$ matrix such that the column vectors $\vec v_1$, $\vec v_2$, ..., $\vec v_{n}$ are orthonormal and the row vectors $\vec u_1$, $\vec u_2$, ..., $\vec u_{n}$ are orthonormal. This means that $\forall i \in \{0, 1, ..., n\}$, $\vec v_{i} \cdot \vec v_{i} = 1$ and $\vec u_{i} \cdot \vec u_{i} = 1$, and that if $i \neq j$, then $\vec v_{i} \cdot \vec v_{j} = 0$ and $\vec u_{i} \cdot \vec u_{j} = 0$.

Then we find that
\begin{align*}
Q^\mathsf{T} Q &= \begin{bmatrix}
\vec v_1^\mathsf{T} & \\
\vec v_2^\mathsf{T} & \\
... & \\
\vec v_{n}^\mathsf{T} & \\
\end{bmatrix} \begin{bmatrix}\vec v_1& \vec v_2& ...& \vec v_{n}\end{bmatrix}\\
&= \begin{bmatrix} \vec v_1 \cdot \vec v_1 & \vec v_1 \cdot \vec v_2 & ... & \vec v_1 \cdot \vec v_{n}&\\
				   \vec v_2 \cdot \vec v_1 & \vec v_2 \cdot \vec v_2 & ... & \vec v_2 \cdot \vec v_{n}&\\
				   ... & ... & ... & ... &\\
				   \vec v_{n} \cdot \vec v_1 & \vec v_{n} \cdot \vec v_2 & ... & \vec v_{n} \cdot \vec v_{n} \end{bmatrix}\\
&= \begin{bmatrix} 1 & 0 & ... & 0&\\
				   0 & 1 & ... & 0&\\
				   ... & ... & ... & ... &\\
				   0 & 0 & ... & 1
\end{bmatrix}\\
&= I
\end{align*}

We also find that
\begin{align*}
Q Q^\mathsf{T} &= \begin{bmatrix}
\vec u_1&\\
\vec u_2&\\
...&\\
\vec u_{n}
\end{bmatrix} \begin{bmatrix}\vec u_1^\mathsf{T}& \vec u_2^\mathsf{T}& ...& \vec u_{n}^\mathsf{T} \end{bmatrix}\\
&= \begin{bmatrix} \vec u_1 \cdot \vec u_1 & \vec u_1 \cdot \vec u_2 & ... & \vec u_1 \cdot \vec u_{n}&\\
				   \vec u_2 \cdot \vec u_1 & \vec u_2 \cdot \vec u_2 & ... & \vec u_2 \cdot \vec u_{n}&\\
				   ... & ... & ... & ... &\\
				   \vec u_{n} \cdot \vec u_1 & \vec u_{n} \cdot \vec u_2 & ... & \vec u_{n} \cdot \vec u_{n} \end{bmatrix}\\
&= \begin{bmatrix} 1 & 0 & ... & 0&\\
				   0 & 1 & ... & 0&\\
				   ... & ... & ... & ... &\\
				   0 & 0 & ... & 1
\end{bmatrix}\\
&= I
\end{align*}

This means that since $Q Q^\mathsf{T} = Q^\mathsf{T} Q = I$,  $Q \in O(n)$.
\end{proof}

\begin{theorem} \label{T:InnerProduct}
Multiplication by a matrix $Q \in O(n)$ preserves the inner product, that is, for any vectors $\vec x, \vec y \in \R^{n}$, $\langle \vec x, \vec y \rangle = \langle Q \vec x, Q \vec y \rangle$.
\end{theorem}
\begin{proof}
First we denote the row vectors of $Q$ as $\vec u_1, \vec u_2, ..., \vec u_{n}$, so that $Q = \begin{bmatrix}\vec u_1&\\ \vec u_2&\\ ...&\\ \vec u_{n} \end{bmatrix}$.

By Theorem \ref{T:ColumnsOrthogonal}, the row vectors $\vec u_1, \vec u_2, ..., \vec u_{n}$ form an orthonormal basis. We then denote the transpose of each $\vec u_{i}$ as $\vec v_{i}$, so $\vec v_{i} = \vec u_{i}^\mathsf{T}$. This means that the $\vec v_{i}$ also form an orthonormal set. It follows that for all $1 \leq i \leq n$, $\langle \vec v_i, \vec v_i \rangle = \vec v_i^\mathsf{T} \vec v_i = 1$ and for all $i \neq j$, $\langle \vec v_i, \vec v_j \rangle = \vec v_i^\mathsf{T} \vec v_j = 0$.

This means that
\begin{align*}
\langle \vec x, \vec y \rangle &= \langle a_1 \vec v_1 + a_2 \vec v_2 + ... + a_{n} \vec v_{n}, b_1 \vec v_1 + b_2 \vec v_2 + ... + b_{n} \vec v_{n} \rangle \\
&= \langle a_1 \vec v_1, b_1 \vec v_1 + b_2 \vec v_2 + ... + b_{n} \vec v_{n} \rangle + \langle a_2 \vec v_2, b_1 \vec v_1 + b_2 \vec v_2 + ... + b_{n} \vec v_{n} \rangle \\& + ... + \langle a_{n} \vec v_{n}, b_1 \vec v_1 + b_2 \vec v_2 + ... + b_{n} \vec v_{n} \rangle \\
&= \langle a_1 \vec v_1, b_1 \vec v_1 \rangle + \langle a_1 \vec v_1, b_2 \vec v_2 \rangle + ... + \langle a_1 \vec v_1, b_{n} \vec v_{n} \rangle \\&+ \langle a_2 \vec v_2, b_1 \vec v_1 \rangle + \langle a_2 \vec v_2, b_2 \vec v_2 \rangle + ... + \langle a_2 \vec v_2, b_{n} \vec v_{n} \rangle \\&+ ... + \langle a_{n} \vec v_{n}, b_1 \vec v_1 \rangle + \langle a_{n} \vec v_{n}, b_2 \vec v_2 \rangle + \langle a_{n} \vec v_{n}, b_{n} \vec v_{n} \rangle \\
&= a_1 b_1 \langle \vec v_1, \vec v_1 \rangle + a_1 b_2 \langle \vec v_1, \vec v_2 \rangle + ... + a_1 b_{n} \langle \vec v_1, \vec v_{n} \rangle \\&+ a_2 b_1 \langle \vec v_2, \vec v_1 \rangle + a_2 b_2 \langle \vec v_2, \vec v_2 \rangle + ... + a_2 b_{n} \langle \vec v_2, \vec v_{n} \rangle \\&+ ... + a_{n} b_1 \langle \vec v_{n}, \vec v_1 \rangle + a_{n} b_2 \langle \vec v_{n}, \vec v_2 \rangle + a_{n} b_{n} \langle \vec v_{n}, \vec v_{n} \rangle \\
&= a_1 b_1 + a_2 b_2 + ... + a_{n} b_{n}
\end{align*}

Furthermore, we see that
\begin{align*}
\langle Q \vec x, Q \vec y \rangle &= \langle Q (a_1 \vec v_1 + a_2 \vec v_2 + ... + a_{n} \vec v_{n}), Q (b_1 \vec v_1 + b_2 \vec v_2 + ... + b_{n} \vec v_{n})\rangle \\
&= \left\langle \begin{bmatrix} u_1 &\\ \vec u_2 &\\ ... &\\ \vec u_{n}\end{bmatrix} (a_1 \vec v_1 + a_2 \vec v_2 +... + a_{n} \vec v_{n}), \begin{bmatrix} u_1 &\\ \vec u_2 &\\ ... &\\ \vec u_{n}\end{bmatrix} (b_1 \vec v_1 + b_2 \vec v_2 + ... + b_{n} \vec v_{n})\right\rangle \\
&= \left\langle \begin{bmatrix} \vec v_1^\mathsf{T} (a_1 \vec v_1 + a_2 \vec v_2 + ... + a_{n} \vec v_{n}) &\\ \vec v_2^\mathsf{T} (a_1 \vec v_1 + a_2 \vec v_2 +... + a_{n} \vec v_{n}) &\\ ... &\\ \vec v_{n}^\mathsf{T} (a_1 \vec v_1 + a_2 \vec v_2 +... + a_{n} \vec v_{n})\end{bmatrix}, \begin{bmatrix} \vec v_{1}^\mathsf{T} (b_1 \vec v_1 + b_2 \vec v_2 + ... + b_{n} \vec v_{n}) &\\ \vec v_{2}^\mathsf{T} (b_1 \vec v_1 + b_2 \vec v_2 + ... + b_{n} \vec v_{n}) &\\ ... &\\ \vec v_{n}^\mathsf{T} (b_1 \vec v_1 + b_2 \vec v_2 + ... + b_{n} \vec v_{n})\end{bmatrix}\right\rangle \\
&= \left\langle \begin{bmatrix} a_1 &\\ a_2 &\\ ... &\\ a_{n}\end{bmatrix}, \begin{bmatrix} b_1 &\\ b_2 &\\ ... &\\ b_{n}\end{bmatrix}\right\rangle \\
&= a_1 b_1 + a_2 b_2 + ... + a_{n} b_{n}
\end{align*}

Combining the above results, we find that
\begin{align*}
\langle \vec x, \vec y \rangle = \langle Q \vec x, Q \vec y \rangle = a_1 b_1 + a_2 b_2 + ... + a_{n} b_{n}
\end{align*}

\end{proof}

\begin{theorem} \label{T:Lengths}
Multiplication by matrices in $O(n)$ preserves the length of vectors with the Euclidean norm, that is, for any vector $\vec x \in \R^{n}$, $||Q \vec x||_2 = ||\vec x||_2$.
\end{theorem}
\begin{proof}
Let $Q \in O(n)$ and $\vec x \in \R^{n}$.

This result follows from Theorem \ref{T:InnerProduct} as follows,
\begin{align*}
\| Q \vec x \|_2 &= \sqrt{\langle Q \vec x, Q \vec x\rangle} \\
&= \sqrt{\langle \vec x, \vec x\rangle} \\
&= \| \vec x \|_2
\end{align*}
\end{proof}

\begin{theorem} \label{T:Angles}
Multiplication by matrices in $O(n)$ preserves the angle between vectors.
\end{theorem}
\begin{proof}
Let $\vec v_1, \vec v_2 \in \R^{n}$, and let $\theta$ be the angle between between $\vec v_1$ and $\vec v_2$. Also let $\theta_{Q}$ be the angle between $Q \vec v_1$ and $Q \vec v_2$. Then the angle $\theta$ satisfies \cite{AppliedLinearAlgebra}
\begin{equation*}
cos(\theta) = \frac{\langle \vec v_1, \vec v_2 \rangle}{\left\| \vec v_1 \right \| \left\| \vec v_2 \right \|}
\end{equation*}

We then find that from the above fact and Theorems \ref{T:InnerProduct} and \ref{T:Lengths}, 
\begin{align*}
cos(\theta_{Q}) &= \frac{\langle Q \vec v_1, Q \vec v_2 \rangle}{\left\| Q \vec v_1 \right \| \left\| Q \vec v_2 \right \|} \\
&= \frac{\langle \vec v_1, \vec v_2 \rangle}{\left\| \vec v_1 \right \| \left\| \vec v_2 \right \|} \\
&= cos(\theta)
\end{align*}

This means that $cos(\theta) = cos(\theta_{Q})$, and it follows that (if without loss of generality we assume $0 \leq \theta < \pi$), $\theta_{Q} = \theta$.
\end{proof}

The group $SO(n)$ corresponds geometrically to rotations of vectors in an n-dimensional space (it is the group of linear transformations preserving length and orientation of vectors), while $O(n)$ corresponds geometrically to improper rotations (preserving length of vectors, but not necessarily orientation). ~\cite{FriedbergInselSpence}

\subsection{Permutation Matrices}

A special class of matrices in $O(n)$ is the set of permutation matrices, which correspond to reordering the columns of the data matrix.

\begin{definition}
A \emph{permutation matrix} is an $n$ by $n$ matrix such that each row contains exactly one 1 and all other entries in that row are 0, and each column contains exactly one 1 and all other entries in that column are zero. \cite{AppliedLinearAlgebra}
\end{definition}

\begin{theorem} \label{T:permutation}
Any permutation matrix is in $O(n)$.
\end{theorem}
\begin{proof}
Let $P$ be an $n$ by $n$ permutation matrix. By the definition of a permutation matrix, each row and each column contains exactly one 1 and the rest of the entries are 0's. This means that the dot product of any column or row vector with itself is 1. However, if $\vec v_{i}$ and $\vec v_{j}$ two distinct column (or row) vectors, then they cannot have a 1 in the same position since otherwise the row (or column) corresponding to that position would have more than one 1, but each row (or column) must have exactly one 1. This means that if $i \neq j$ and $v_{i}$, $v_{j}$ are column (or row) vectors, then $\vec v_{i} \cdot \vec v_{j} = 0$.

Let $\vec v_{i}$ be the column vectors of $P$, so that $P = \begin{bmatrix}\vec v_1& \vec v_2& ...& \vec v_{n}\end{bmatrix}$. Then $\forall i \in \{0, 1, ..., n\}$, $\vec v_{i} \cdot \vec v_{i} = 1$, and if $i \neq j$, then $\vec v_{i} \cdot \vec v_{j} = 0$.

Then we have that
\begin{align*}
P^\mathsf{T} P &= \begin{bmatrix} \vec v_1^\mathsf{T}&\\ \vec v_2^\mathsf{T}&\\ ...&\\ \vec v_{n}^\mathsf{T} \end{bmatrix} \begin{bmatrix} \vec v_1& \vec v_2& ...& \vec v_{n} \end{bmatrix} \\
&= \begin{bmatrix} \vec v_1^\mathsf{T} \vec v_1 & \vec v_1^\mathsf{T} \vec v_2 & ... & \vec v_1^\mathsf{T} \vec v_{n}&\\
				   \vec v_2^\mathsf{T} \vec v_1 & \vec v_2^\mathsf{T} \vec v_2 & ... & \vec v_2^\mathsf{T} \vec v_{n}&\\
				   ... & ... & ... & ... &\\
				   \vec v_{n}^\mathsf{T} \vec v_1 & \vec v_{n}^\mathsf{T} \vec v_2 & ... & \vec v_{n}^\mathsf{T} \vec v_{n}
\end{bmatrix}\\
&= \begin{bmatrix} 1 & 0 & ... & 0&\\
				   0 & 1 & ... & 0&\\
				   ... & ... & ... & ... &\\
				   0 & 0 & ... & 1
\end{bmatrix}\\
&= I
\end{align*}

Let $\vec u_{i}$ be the column vectors of $P$, so that $P = \begin{bmatrix}\vec u_1&\\ \vec u_2&\\ ...&\\ \vec u_{n}\end{bmatrix}$. Then $\forall i \in \{0, 1, ..., n\}$, $\vec u_{i} \cdot \vec u_{i} = 1$, and if $i \neq j$, then $\vec u_{i} \cdot \vec u_{j} = 0$.

We find that
\begin{align*}
P P^\mathsf{T} &= \begin{bmatrix} \vec u_1&\\ \vec u_2&\\ ...&\\ \vec u_{n} \end{bmatrix} \begin{bmatrix} \vec u_1^\mathsf{T}& \vec u_2^\mathsf{T}& ...& \vec u_{n}^\mathsf{T} \end{bmatrix} \\
&= \begin{bmatrix} \vec u_1 \vec u_1^\mathsf{T} & \vec u_2 \vec u_1^\mathsf{T} & ... & \vec u_{n} \vec u_1^\mathsf{T}&\\
				   \vec u_2^\mathsf{T} \vec u_1 & \vec u_2^\mathsf{T} \vec u_2 & ... & \vec u_2^\mathsf{T} \vec u_{n}&\\
				   ... & ... & ... & ... &\\
				   \vec u_1 \vec u_{n}^\mathsf{T} & \vec u_2 \vec u_{n}^\mathsf{T} & ... & \vec u_{n} \vec u_{n}^\mathsf{T}
\end{bmatrix}\\
&= \begin{bmatrix} 1 & 0 & ... & 0&\\
				   0 & 1 & ... & 0&\\
				   ... & ... & ... & ... &\\
				   0 & 0 & ... & 1
\end{bmatrix}\\
&= I
\end{align*}

This means that $P^\mathsf{T} P = P P^\mathsf{T} = I$, so that $P \in O(n)$.
\end{proof}

Not every permutation matrix is in $SO(n)$ however, for an example consider 

$P = \begin{bmatrix} 0 & 1&\\
			    1 & 0
\end{bmatrix}$

This is a permutation matrix, but $det(P) = -1$, and so $P \not\in SO(2)$.

\subsection{Applying Orthogonal Matrices in Borel Isomorphic Dimensionality Reduction}

One approach is to generate random matrices in $O(n)$ or $SO(n)$ and see which ones improve the accuracy. To generate a random matrix in $O(n)$, we start with a random invertible matrix (since the set of singular matrices in the set of $n$ by $n$ matrices has measure zero, almost all random matrices are invertible, which means that we can simply start with a random matrix) and then we can apply the Gram-Schmidt algorithm to orthonormalize the columns. To generate a matrix in $SO(n)$, we check the sign of the determinant and flip the sign of the first row if the determinant is negative, this will create an orthonormal matrix with a determinant of 1, which is therefore in $SO(n)$. 

One problem that we have is that multiplying the data by an arbitrary matrix in $O(n)$ or even in $SO(n)$ does not always keep the data in $[0, 1]$.
\begin{example}
For instance, let
\begin{equation}
Q = \begin{bmatrix}
\frac{1}{\sqrt{2}} & \frac{1}{\sqrt{2}} & \\
-\frac{1}{\sqrt{2}} & \frac{1}{\sqrt{2}}
\end{bmatrix}
\end{equation}

We can check that $Q^\mathsf{T} Q = Q Q^\mathsf{T} = I$ and $det(Q) = 1$, so $Q \in SO(n)$.

Suppose one row of the data matrix is $\left(0.7,\;0.9\right)$. All entries of this row are in $[0, 1]$. However, when this row is multiplied by $Q$, we find that
\begin{align*}
x Q &= \begin{bmatrix}0.7 & 0.9\end{bmatrix} \begin{bmatrix}
\frac{1}{\sqrt{2}} & \frac{1}{\sqrt{2}} & \\
-\frac{1}{\sqrt{2}} & \frac{1}{\sqrt{2}}
\end{bmatrix} \\
&= \begin{bmatrix}-\frac{\sqrt{2}}{10} & \frac{\sqrt{2}}{1.25}\end{bmatrix} \\
&\approx \left(-0.141, 1.131\right)
\end{align*}

Obviously neither entry of $\left(-0.141, 1.131\right)$ is in $[0, 1]$, in fact one entry is negative while the other is greater than one. This is a problem as the digit interchange Borel isomorphism assumes that the values are in $[0, 1]$.
\end{example}

We can solve this problem by applying the Borel isomorphic function (from Theorem \ref{T:normalize})
\begin{equation}
f(x) = \frac{x - x_{min}}{x_{max} - x_{min}}
\end{equation}
to the matrix after multiplication. Since the composition of Borel isomorphisms is a Borel isomorphism (by Theorem \ref{T:composition}), this means that applying this function after the matrix multiplication maintains the Borel isomorphism.

Another interesting approach is to consider just one column of the data set and the result column, and to rank the columns in how good the column by itself is at predicting the outcome. We then select, if all the dimensions will be reduced into one, that the first column to be the most accurate at predicting the outcome by itself, the second column to be the second most accurate, and so on. This means that columns that are the best predictors are given larger weight. In the event that there will be multiple dimensions in the output data set, say 4 dimensions, we compress columns ranked 1, 5, 9, 13, ... into one (in that order), columns ranked 2, 6, 10, 14, ... into one, columns ranked 3, 7, 11, 15, ... into one, and columns ranked 4, 8, 12, 16, ... into one. We can then generate a  permutation matrix corresponding to ordering the matrices this way. Since this is a permutation matrix, it is in $O(n)$ by Theorem \ref{T:permutation}, and so we can multiply the data set by it before applying the digit interlace Borel isomorphism.

\subsection{Results}

We test this with the Yeast dataset. We generate the identity matrix (as a control), permutation matrix, 100 random $SO(n)$ matrices, and 100 random $O(n)$ matrices. We then observe the accuracy of each of these and obtain the results summarized in table \ref{F:YeastMatrixTable}.

\begin{table}[ht]
\centering
\begin{tabular}{|r|r|r|}
  \hline
Matrix & Mean Accuracy & Standard Deviation of Accuracy \\ 
  \hline
Identity matrix & 0.51 & 0.02 \\ 
  \hline
  Permutation matrix & 0.55 & 0.01 \\ 
    \hline
  Best random SO(n) matrix & 0.50 & 0.01 \\ 
    \hline
  Best random O(n) matrix & 0.52 & 0.01 \\ 
   \hline
\end{tabular}
\caption{A table showing the accuracy of kNN applied to the Yeast dataset after multiplication by various matrices.}
\label{F:YeastMatrixTable}
\end{table}

The results are summarized in the box-and-whiskers plot \ref{F:YeastMatrixResult}.

\begin{figure}[H]
\includegraphics[width=\textwidth]{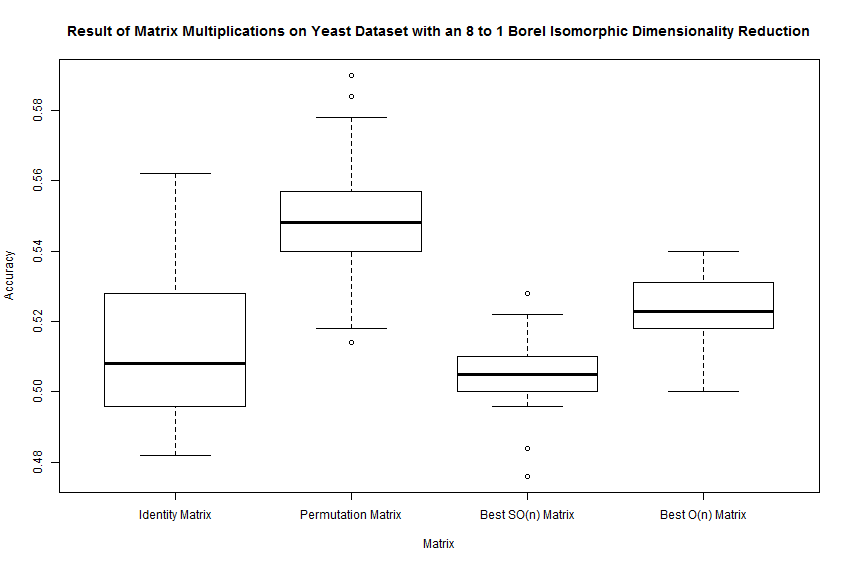}
\caption{Comparison of matrix multiplications with 8 to 1 Borel isomorphic reduction with Yeast dataset.}
\label{F:YeastMatrixResult}
\end{figure}

We observe that generating a single permutation matrix has an accuracy better than the best random $O(n)$ or $SO(n)$ matrix. This suggests that there is no need to randomly generate matrices and that simply using the permutation matrix (generated by ranking the columns by accuracy and placing more columns that generate more accuracy predictions at the front) is the best approach for the Yeast dataset.

When we do a 16 to 1 Borel isomorphic reduction on the Phoneme dataset, we find that multiplication by the permutation matrix \emph{decreases} the accuracy, while multiplication by matrices in $O(n)$ and $SO(n)$ increases the accuracy, as seen in figure \ref{F:PhonemeMatrixResult}.

\begin{figure}[H]
\includegraphics[width=\textwidth]{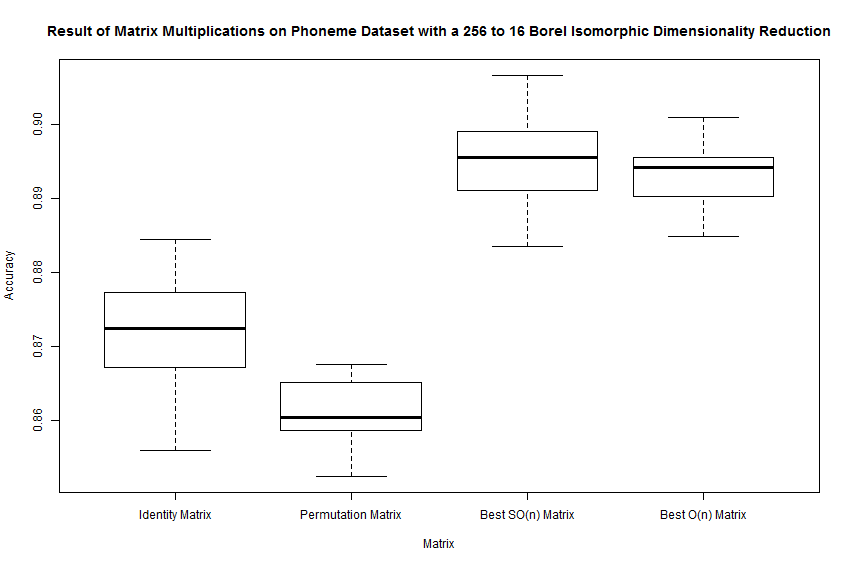} 
\caption{Comparison of matrix multiplications with 256 to 16 Borel isomorphic reduction with Phoneme dataset.}
\label{F:PhonemeMatrixResult}
\end{figure}

However, if we do a 256 to 1 Borel isomorphic reduction on the Phoneme dataset, we find that multiplication by the permutation matrix slightly \emph{increases} the accuracy, and multiplication by random matrices in $O(n)$ and $SO(n)$ increased the accuracy even more, as we see in figure \ref{F:PhonemeMatrixResult2}. We see that multiplication by matrices in $O(n)$ or $SO(n)$ is a good approach for the phoneme dataset, as it produces better results than the identity matrix and the permutation matrix.

\begin{figure}[H]
\includegraphics[width=\textwidth]{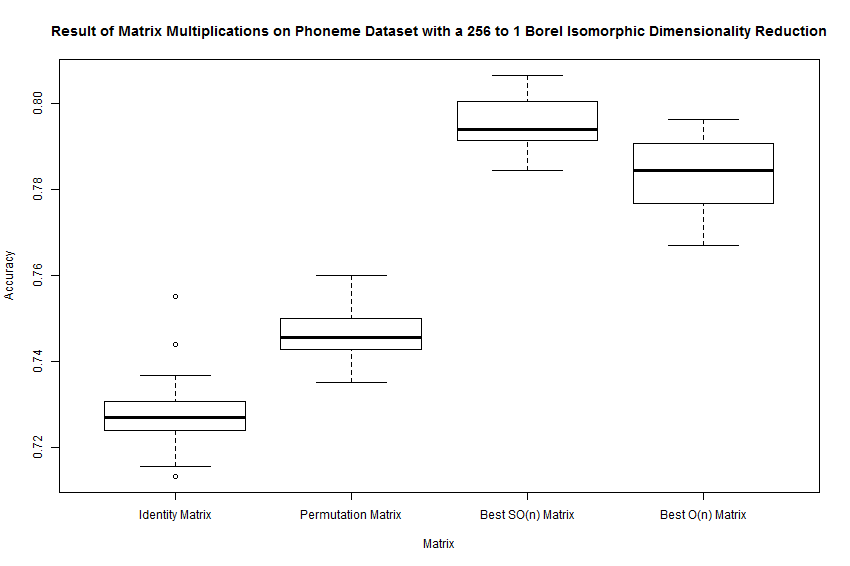} 
\caption{Comparison of matrix multiplications with 256 to 1 Borel isomorphic reduction with Phoneme dataset.}
\label{F:PhonemeMatrixResult2}
\end{figure}

\subsection{Effect on Base}

We would like to see how multiplying by a matrix in $O(n)$ affects the optimal base for the Borel isomorphism.

We test this with the Yeast dataset and the permutation matrix. When we try all integer bases between 2 and 30 after multiplying by the permutation matrix and applying the Borel isomorphism, we obtain the results in figure \ref{F:YeastMatrixBase}.

\begin{figure}[H]
\includegraphics[width=\textwidth]{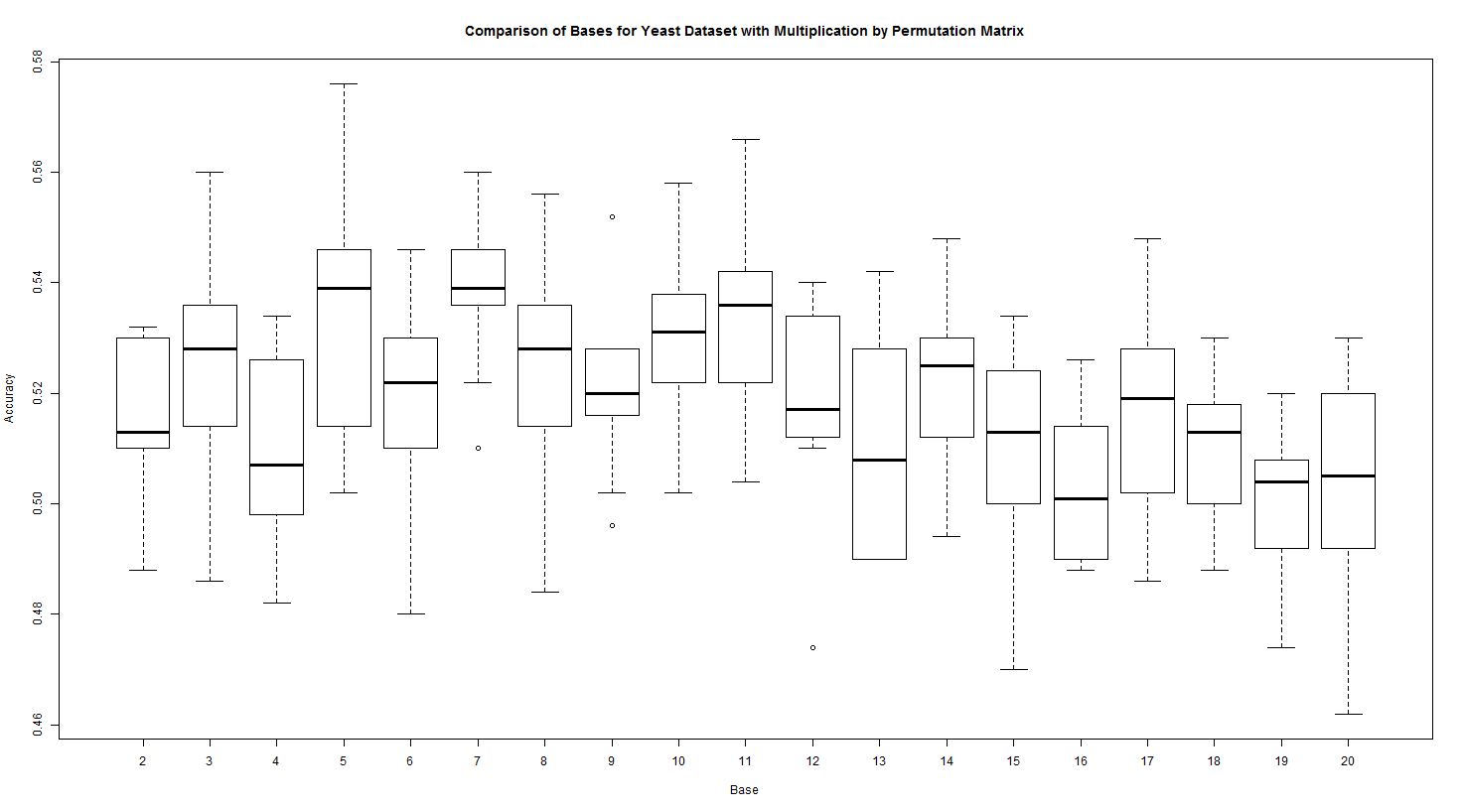}
\caption{Comparison of various bases for the Borel isomorphism after multiplication by a permutation matrix for the Yeast dataset.}
\label{F:YeastMatrixBase}
\end{figure}

We see that the the optimal bases appear to be base 5 and base 7 after the data set was multiplied by the permutation matrix.

\section{Future Prospects}

In this project we have been using only the k-NN classifier. A future project would be to try this with other classifiers, including support vector machine (SVM) and random forest. The familiar problem with random forests is that they are not universally consistent, however they work very well in practice. We would like to verify the accuracy of random forests after a Borel isomorphism is applied as well.

Here is what we see as the paramount direction for research. Having at hand a large family of Borel isomorphisms to choose from will definitely help to improve the accuracy, as we have seen on many examples. On the other hand, the question that needs to be asked is the following: how to make sure that this family is not too large to avoid the problem of overfitting, because for instance if every Borel isomorphism is allowed we can fit such an isomorphism to data perfectly, but there is no guarantee that a new incoming data point will be classified correctly. This is a very familiar problem in statistics and in statistical machine learning. In machine learning it is normally solved by devising a certain measure of complexity of a class, such as the VC dimension or more generally metric entropy. We believe that it's a very interesting and important direction for research to devise a similar measure of complexity for classes of Borel isomorphisms and deduce estimates for the guaranteed accuracy.


\begin{thebibliography}{999}

\bibitem{Kechris} Kechris, A., \emph{Classical Descriptive Set Theory}, Springer-Verlag, 1995.

\bibitem{Berberian} Berberian, S.K., \emph{Borel Sets}, 1988.

\bibitem{Pestov} Pestov, V., \emph{Is the k-NN classifier in high dimensions affected by the curse of dimensionality?}, Computers and Mathematics with Applications (2012), doi:10.1016/j.camwa.2012.09.011

\bibitem{Stone} Stone, C., \emph{Consistent Nonparametric Regression}, Annals of Statistics, 1977.

\bibitem{FriedbergInselSpence} Stephen H. Friedberg, Arnold J. Insel, and Lawrence E. Spence, \emph{Linear Algebra}, Pearson Prentice Hall, 2003.

\bibitem{DiscontinousLinearMap} \url{http://en.wikipedia.org/wiki/Discontinuous_linear_map#If_a_linear_map_is_finite_dimensional.2C_the_linear_map_is_continuous}

\bibitem{PermutationMatrix} \url{http://mathworld.wolfram.com/PermutationMatrix.html}

\bibitem{AppliedLinearAlgebra} Ben Noble, James W. Daniel, \emph{Applied Linear Algebra} (Third Edition), Prentice Hall, 1987

\bibitem{PolishSpacesUpToBorelIsomorphism} \url{http://planetmath.org/polishspacesuptoborelisomorphism}

\bibitem{Strichartz} Robert S. Strichartz, \emph{The Way of Analysis}, Jones and Bartlett Learning, 2000.

\bibitem{Folland} Gerald B. Folland, \emph{Real Analysis: Modern Techniques and Their Applications}, Second Edition (Pure and Applied Mathematics: A Wiley Series of Texts, Monographs and Tracts), Wiley, 1999.

\bibitem{Bruckner} Andrew M. Bruckner, Judith B. Bruckner, Brian S. Thomson, \emph{Real Analysis}, Second Edition, ClassicalRealAnalysis.com, 2008, xiv 656 pp. [ISBN 1434844129]

\bibitem{McDonald} David McDonald, \emph{Elements of Applied Probability for Engineering, Mathematics, and Systems Science}, World Scientfic, 2004

\bibitem{Nicholson}  W. Keith Nicholson, \emph{Introduction to Abstract Algebra}, Third Edition, Wiley-Interscience, 2006. 

\bibitem{Cohen} Graeme Cohen, \emph{A Course in Modern Analysis and its Applications}, Cambridge University Press, 2003

\bibitem{Dudley} Dudley, R. M., \emph{Real Analysis and Probability} (Cambridge Studies in Advanced Mathematics), Cambridge University Press, 2004

\bibitem{pbook} Luc Devroye, L\'{a}szl\'{o} Gy\"{o}rfi, G\'{a}bor Lugosi, \emph{A Probabilistic Theory of Pattern Recognition} (Stochastic Modelling and Applied Probability), Springer, 1996

\bibitem{Srivastava} Srivastava, S.M., \emph{A Course on Borel Sets}, (Graduate Texts in Mathematics), Springer, 2003

\bibitem{UCI} Bache, K. \& Lichman, M. (2013). UCI Machine Learning Repository \url{http://archive.ics.uci.edu/ml}. Irvine, CA: University of California, School of Information and Computer Science.

\bibitem{yeast} Kenta Nakai and Paul Horton, Yeast Dataset \url{http://archive.ics.uci.edu/ml/datasets/Yeast}

\bibitem{phoneme} Andreas Buja, Werner Stuetzle and Martin Maechler, Phoneme Dataset \url{http://orange.biolab.si/datasets/phoneme.htm}

\bibitem{Stewart} William J. Stewart, \emph{Probability, Markov Chains, Queues, and Simulation: The Mathematical Bases of Performance Modeling}, Princeton University Press, 2009.

\bibitem{Stein} Elias M. Stein and Rami Shakarchi, \emph{Real Analysis: Measure Theory, Integration, and Hilbert Spaces} (Princeton Lectures in Analysis), Princeton University Press, 2005

\end{thebibliography}
\end{document}